\documentclass{article}

\usepackage{authblk}

\usepackage[dvipsnames]{xcolor}
\usepackage{subfigure}
\usepackage{graphicx}
\usepackage{appendix}
\usepackage[colorlinks=true,linkcolor=blue,citecolor=blue]{hyperref}

\usepackage{natbib}
\usepackage{amsmath,amsthm,amssymb}
\usepackage{cleveref}
\usepackage{thmtools,thm-restate}
\usepackage{caption}
\usepackage{mathtools}
\usepackage{xfrac}
\usepackage{enumitem}

\newcommand{\suchthat}{\;\ifnum\currentgrouptype=16 \middle\fi|\;}

\newcommand{\R}{\mathbb{R}}

\newcommand{\Reg}{\text{Reg}}

\newcommand{\hatf}{\hat{f}}

\newcommand{\F}{\mathcal{F}}

\newcommand{\E}{\mathop{\mathbb{E}}}
\newcommand{\A}{\mathcal{A}}

\newcommand{\Xscript}{\mathcal{X}}
\newcommand{\ordO}{\mathcal{O}}

\newcommand{\Unif}{\text{Unif}}
\newcommand{\I}{\mathbb{I}}

\newcommand{\G}{\mathcal{G}}

\newcommand{\comp}{\textbf{comp}}

\newcommand{\msafe}{m^*}
\newcommand{\msafealg}{\hat{m}}

\newcommand{\eventReg}{\mathcal{W}_1}
\newcommand{\eventEval}{\mathcal{W}_2}
\newcommand{\eventEvalRegret}{\mathcal{W}_3}
\newcommand{\eventEvalcheck}{\mathcal{W}_4}
\newcommand{\eventAzuma}{\mathcal{W}_5}

\newcommand{\TsafeCheck}{\mathcal{T}}
\newcommand{\N}{\mathbb{N}}

\newcommand{\safe}{\textbf{safe}}
\newcommand{\true}{\textbf{True}}
\newcommand{\SafeFalcon}{\text{Safe-FALCON}}
\newcommand{\IsSafe}{\text{Check-is-safe}}
\newcommand{\ChoosePi}{\text{Choose-safe}}
\newcommand{\CumulativeReward}{\text{Crwd}}

\newcommand{\kernelSet}{\mathcal{K}(\F)}

\newtheorem{assumption}{Assumption}

\newtheorem{lemma}{Lemma}

\graphicspath{{figures/}}

\usepackage{microtype}
\usepackage{graphicx}
\usepackage{subfigure}
\usepackage{booktabs} 

\usepackage{hyperref}

\usepackage[accepted]{icml2021}

\icmltitlerunning{Adapting to Misspecification in Contextual Bandits}

\begin{document}

\twocolumn[
\icmltitle{Adapting to Misspecification in Contextual Bandits with Offline Regression Oracles}



\icmlsetsymbol{equal}{*}

\begin{icmlauthorlist}
\icmlauthor{Sanath Kumar Krishnamurthy}{to}
\icmlauthor{Vitor Hadad}{goo}
\icmlauthor{Susan Athey}{goo}
\end{icmlauthorlist}

\icmlaffiliation{to}{Management Science and Engineering, Stanford University, Stanford, CA, USA}
\icmlaffiliation{goo}{Graduate School of Business, Stanford University, Stanford, CA, USA}

\icmlcorrespondingauthor{Sanath Kumar Krishnamurthy}{sanathsk@stanford.edu}

\icmlkeywords{Contextual Bandits}

\vskip 0.3in
]



\printAffiliationsAndNotice{}  

\begin{abstract}
    Computationally efficient contextual bandits are often based on estimating a predictive model of rewards given contexts and arms using past data.  However, when the reward model is not well-specified, the bandit algorithm may incur unexpected regret, so recent work has focused on algorithms that are robust to misspecification. We propose a simple family of contextual bandit algorithms that adapt to misspecification error by reverting to a good safe policy when there is evidence that misspecification is causing a regret increase. Our algorithm requires only an offline regression oracle to ensure regret guarantees that gracefully degrade in terms of a measure of the average misspecification level. Compared to prior work, we attain similar regret guarantees, but we do no rely on a master algorithm, and do not require more robust oracles like online or constrained regression oracles  (e.g., \cite{foster2020adapting}; \cite{krishnamurthy2020tractable}). This allows us to design algorithms for more general function approximation classes.
\end{abstract}

\section{Introduction}
\label{sec:introduction}

Contextual bandit algorithms are a fundamental tool in sequential decision making and have been used in a variety of applications \citep[see e.g.,][Section 1.4 for a review]{lattimore2020bandit}. 

The finite-armed (stochastic) contextual bandit setting that this paper is concerned with can be described as follows. Over a sequence of rounds, a bandit algorithm receives some side information or ``contexts'', which is drawn from a fixed distribution. Upon receiving each context, the algorithm selects an action, and then receives a probabilistic reward whose distribution may depend on the context and action. The objective of the algorithm is to interactively learn a mapping from contexts to actions so as to maximize the rewards received during the experiment. In order to do so, it must efficiently trade off the need for resolving uncertainty about the value of each action (exploration), with the objective of maximizing rewards (exploitation).

Many contextual bandit algorithms make use of an estimate of the conditional mean function of rewards given contexts and arms, along with some measure of the uncertainty around this function. At a high level, if the model predicts that an action has high expected reward with low uncertainty, then the algorithm is more likely to select this action. This intuition leads to heuristics that are statistically optimal in some settings \citep[e.g.,][]{agrawal2013thompson, li2010contextual}. Such algorithms also tend to be computationally tractable relative to alternatives \citep{agarwal2014taming}, since all that is required is a predictive model and the ability to produce appropriate confidence intervals.

However, the success of such algorithms depends heavily on tenuous assumptions about the underlying data-generating process, as their performance guarantees often rely on the conditional mean function belonging to a particular class; e.g., that it be linear under some transformation of the contexts. This is often called ``realizability'' in the literature, and when violated can cause the algorithm to behave erratically. 

In this paper, we suggest an algorithm that adapts to model misspecification. To illustrate the problem, we consider an example described in \cite{krishnamurthy2020tractable} and use it to get insights into the behavior of a realizability-based algorithm \texttt{FALCON+} \citep{simchi2020bypassing} in a setting where realizability does not hold. Consider a two-arm contextual bandit setting:
\begin{equation}
    \label{eq:intro_example_model}
    \begin{aligned}
        \E[r_t \big| x_t = x, a_t = a] 
        =
        \begin{cases}
            \mathbb{I}\{ x_t > 0.5 \} &\text{if a = 1} \\
            0.5 &\text{if a = 2} \\
        \end{cases}
    \end{aligned}
\end{equation}
where $x_t \sim \text{Uniform}[0,1]$ represents the contexts observed at the beginning of the $t$-th round, $a_t$ is the action taken, $r_t$ is the reward, which is observed by the experimenter with noise $e_t \sim \mathcal{N}(0, 1)$. Although the conditional expectation of rewards \eqref{eq:intro_example_model} is clearly non-linear, the model is simple enough that one could expect a bandit algorithm based on a misspecified linear model to do well. 

Figure \ref{fig:falcon_failure} shows the behavior of average regret for a realizability-based algorithm \texttt{FALCON+} \citep{simchi2020bypassing}, under the (incorrect) assumption that the underlying model is linear. The details of this algorithm are not particularly important. It suffices to understand that at the end of approximately every $2^{m}$ rounds (we call these intervals ``epochs'' and index them by $m$), the algorithm computed an estimate $\hat{f}(x, a)$ of \eqref{eq:intro_example_model}, assuming a linear model and based on data from the previous epoch. Then, for every round in this epoch, it selects arms based on a probabilistic model where arms with high $\hat{f}(x, \cdot)$ have higher probability. A full description of this example is given in the appendix.

\begin{figure}[t]
    \centering
    \includegraphics[width=0.44\textwidth]{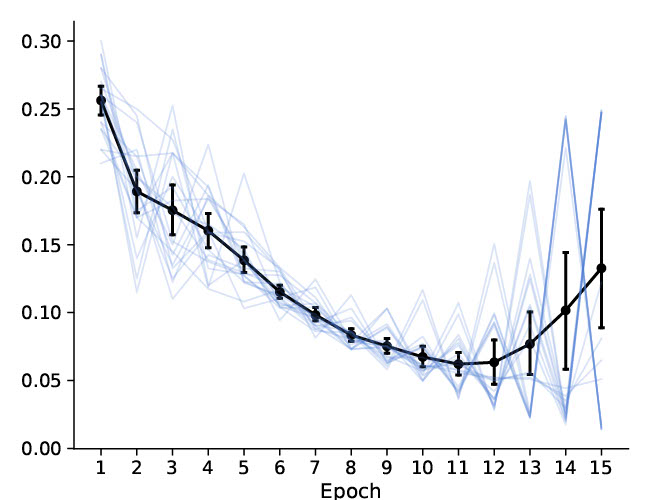}
    \caption{Illustration of the failure of a contextual bandit algorithm (\texttt{FALCON+}) when the model is not well-specified. Each epoch starts at round $2^m$. Vertical bars are 95\% confidence intervals around the average per-epoch average regret, aggregated over 50 simulations. Spaghetti plots are average per-epoch regret for 20 representative simulations.}
    \label{fig:falcon_failure}
\end{figure}

We notice two phenomena. First, the spaghetti plot reveals that average per-epoch reward has an oscillatory behavior, switching between very low and very high levels. This can be explained as follows. In some of the later epochs, the algorithm estimates a model that is close to the best linear approximation of \eqref{eq:intro_example_model}, and thus it almost exclusively selects actions optimally (i.e., $a = 1$ when $x_t > .5$, $a = 2$ otherwise). In such epochs average rewards are high. However, in doing so, it collects data that is so skewed that it adversely affects the model estimated in the next epoch, causing the algorithm to make many mistakes and driving average rewards down again. However, these mistakes in turn allow the algorithm to get less skewed data for the misspecified arm, leading to estimate a good linear approximation to \eqref{eq:intro_example_model} again, and the cycle repeats.

More importantly, as a consequence of the erratic behavior just described we observe a second phenomenon: although the average regret for decreases initially, it begins to increase again after some time. The fundamental reason for this failure is that, by ignoring misspecification, the algorithm fails to accurately capture the how much uncertainty there is about the true model, in particular in regions where there is heavy extrapolation. The resulting performance is clearly suboptimal, and in this case, the model estimates do not even seem to converge. This problem is not idiosyncratic to \texttt{FALCON+}. For example,  \cite{krishnamurthy2020tractable} shows that LinUCB \citep{li2010contextual} can converge to a suboptimal solution under the same example.

In this work, we propose a method to prevent the undesirable behavior described above. Our starting point is the \texttt{FALCON+} algorithm. We modify it by introducing an additional step in which we test for a dip in average rewards (i.e., an increase in regret) caused by model misspecification. Upon finding evidence of this issue, we revert to a previously estimated ``safe'' policy and reduce further model updates. When there is no model misspecification, we attain the optimal regret guarantees inherited from \texttt{FALCON+}. When there is model misspecification, our regret bounds have an additional term that depends on a measure of misspecification.


Our results bound the regret overhead due to misspecification by $\ordO(\epsilon\sqrt{K}T)$, where $\epsilon$ is the average misspecification error. Roughly speaking, we define the average misspecification error as a tight upper bound on the root mean squared difference between the true model and any function in the class of posited model that minimizes this squared difference (e.g., linear models). See \Cref{sec:Theory} for a formal definition. 



We now briefly describe how we get this result. Our starting point is \texttt{FALCON+}, which runs in epochs denoted by $m$. At each epoch $m$, the amount of exploitation is controlled by a parameter $\gamma_m^{-1}$. This parameter decreases with each epoch, as the algorithm learns about the environment and exploitation increases.
We observe that while $\gamma_m^{-1}$ is larger than the average misspecification $\epsilon$ -- which is true for the initial epochs -- 
cumulative regret decreases at the same rate as when realizability holds. When the average misspecification is comparable to this measure of exploration, we get that the expected instantaneous regret can be bounded by $\ordO(\epsilon\sqrt{K})$. Once $\gamma_m^{-1}$ becomes less than the average misspecification, the expected instantaneous regret may increase. In fact, this behaviour is observed in \Cref{fig:falcon_failure}. 

Our algorithm works by continuously testing for an unexpected change in cumulative rewards, and when that is detected the algorithm reverts to a good historically ``safe''. Our algorithm achieves the required regret bound because the expected instantaneous regret of this ``safe'' policy is bounded by $\ordO(\epsilon\sqrt{K})$ with high probability. Finally, to ensure that the regret guarantees in the initial epochs continue to hold with the high-probability parameter of our choice, our algorithm uses a $\gamma_m$ parameter that is about $\sqrt{2}$ times smaller compared to the one used in \texttt{FALCON+}.

Our algorithm is computationally efficient and flexible, as all that is needed is an offline regression oracle (i.e., the ability to fit a predictive model), thereby extending the reduction from contextual bandits to offline regression oracles \citep{simchi2020bypassing} to scenarios where realizability may not hold. Further our algorithm does not require knowledge of the average misspecification error, nor does it require the use of master algorithms. 

Not relying on master algorithms to adapt to unknown misspecification allows us enjoy additional computational and statistical benefits. The computational benefit comes from not requiring to maintain and update $\lfloor\log(T)\rfloor$ base bandit algorithms\footnote{These base algorithms make different guesses for the misspecification measure, and the master algorithm chooses the best performing base algorithm.}. We also inherit the optimal realizability-based regret bound from \texttt{FALCON+} when the assumption holds, and save an additional $\ordO(\sqrt{\log(T)})$ factor that would have been incurred had we relied on a master with $\lfloor\log(T)\rfloor$ base algorithms.

Finally, our bounds on regret overhead due to misspecification are in terms of the average misspecification error and match the best known bounds from prior work \citep{foster2020adapting}. 

Our upper bound results are summarized in \Cref{sec:main_results}. To see that these upper bounds are optimal for contextual bandit algorithms that are based on regression oracles, up to constant factors, we also obtain matching lower bounds for the regret overhead due to misspecification (see \Cref{sec:lower-bound}). These results can also be interpreted as a quantification of the bias-variance trade-off for contextual bandits with regression oracles, see \Cref{sec:main_results} for a more detailed discussion.

\subsection{Related work}
\label{sec:lit}


Over the last couple of decades, contextual bandit algorithms have been extensively studied \cite{lattimore2020bandit}. However, the performance of many algorithms rely on the ``realizability'' assumption, which requires the analyst to know the form of the conditionally expected reward model. Moreover, the theoretical analysis of these algorithms that bounded cumulative regret often break down even under mild violations of the assumption. Since the realizability assumption may not be realistic, bandit algorithms that are robust to model misspecification have become a subject of intense recent interest. 

In particular, recent works study algorithms that bound the regret overhead due to misspecification. When the true reward function is linear up to an additive error uniformly bounded by $\epsilon$ and contexts are $d$-dimensional, \cite{neu2020efficient} provide an algorithm that bound the regret overhead due to misspecification by $\ordO(\epsilon\sqrt{d}T)$. This paper assumes perfect knowledge of the covariance matrix for the distribution of contexts, but the algorithm does not require knowledge of $\epsilon$.

This result is improved\footnote{Assuming $K<d$, which is often true.} and generalized by \cite{foster2020beyond}. Given an online regression oracle for a class of value functions $\F$ and suppose the true reward function can be approximated by a function in $\F$ up to an additive error uniformly bounded by $\epsilon$, \cite{foster2020beyond} provide an algorithm that achieves a regret overhead bound of $\ordO(\epsilon\sqrt{K}T)$, where $K$ is the number of arms. The algorithms proposed in this paper uses $\epsilon$ as an input parameter.

In many scenarios, one would expect a uniform bound on the additive error to be rather stringent. Concurrent work bound the regret overhead in terms of the ``average misspecification error'' \cite{foster2020adapting, krishnamurthy2020tractable}, the notions of misspecification used in these papers are not the same but are similar. Roughly speaking, the average misspecification error is averaged over contexts but are uniformly bounded over policies. If $\epsilon$ is the average misspecification error, \cite{foster2020adapting} and \cite{krishnamurthy2020tractable} achieve regret overhead bounds of $\ordO(\epsilon\sqrt{K}T)$ and $\ordO(K^{2/5}\epsilon^{4/5}T)$. 
\footnote{These bounds are non-trivial only if $\epsilon\sqrt{K}<1$, clearly under this setting the bounds guaranteed by \cite{krishnamurthy2020tractable} are weaker. This is because, their algorithm performs uniform sampling for a fraction of the time-steps. }
The algorithms used in these papers assume access to robust regression oracles, like online regression oracles \citep{foster2020adapting} and offline constrained regression oracles \citep{krishnamurthy2020tractable}. Furthermore, \cite{krishnamurthy2020tractable} assume that the set $\F$ is convex and require knowledge of $\epsilon$ (up to a constant factor) as an input parameter. In contrast, \cite{foster2020adapting} can adapt to the unknown misspecification without requiring any information about the misspecification parameter ($\epsilon$). 

To start with, \cite{foster2020adapting} provide a base algorithm that requires knowledge of $\epsilon$ (up to a constant factor) to achieve the regret overhead bound of $\ordO(\epsilon\sqrt{K}T)$. They then consider $\lfloor\log(T)\rfloor$ base algorithms with different guesses for the average misspecification error ($\epsilon$). Finally, they show that a master algorithm can be used to select the best performing base algorithm while continuing to achieve an overall regret overhead bound of $\ordO(\epsilon\sqrt{K}T)$.

The idea of using master algorithms to adapt to unknown misspecification has also been used earlier in the context of misspecified linear bandits \cite{pacchiano2020model}. The misspecified linear bandit setup has also been studied in \cite{ghosh2017misspecified} and \cite{lattimore2020learning}. These papers bound regret overhead in terms of the uniform misspecification error.

Oracle-based agnostic contextual bandit algorithms do not assume realizability and hence directly adapt to unknown misspecification\footnote{Here misspecification would correspond to the optimal policy not lying in the set of policies being explored.} \citep{dudik2011efficient, agarwal2014taming}. Unfortunately, these approaches suffer from computational issues that limit their implementability. 

Within the broader literature of contextual bandits, we build on the recent line of work that provide reductions to offline/online squared loss regression \citep{foster2018practical, foster2020beyond, xu2020upper, foster2020instance}. In particular, our work can be viewed as an extension of the analysis of \cite{simchi2020bypassing} to general scenarios that do not assume realizability.

\section{Theory}
\label{sec:Theory}

To formalize the problem and discuss the properties of our algorithm, let's first establish some basic notation. Other symbols will be introduced later as appropriate.

\paragraph{Basic notation} We let $\A$ denote the finite set of actions, $K$ denote the number of arms (i.e. $K:=|\A|$), and $\Xscript$ denote the set of contexts.  We let the notation $[n]$ denote the set $\{1,...,n \}$. The (possibly unknown) number of rounds is denoted by $T$. Our algorithm will work in epochs indexed by $m$; the final round of each epoch is denoted by $\tau_m$. We let $m(t)$ denote the epoch containing round $t$ -- that is, $m(t):=\min \{m| t \leq \tau_m \}$.

At every time-step $t \in [T]$, the environment draws a context $x_t\in\Xscript$ and reward vector $\smash{r_t \in [0,1]^K}$ from a fixed but unknown distribution $D$. Unless stated otherwise, all expectations are with respect to this distribution. Using potential outcome notation, we let $r_t(a)$ denote the reward that associated with arm~$a$ at time $t$. 

We use $p$ to denote probability kernels from $\A\times\Xscript$ to $[0,1]$, and let $D(p)$ be the induced distribution over $\Xscript\times\A\times[0,1]$, where sampling $(x,a,r(a))\sim D(p)$ is equivalent to sampling $(x,r)\sim D$ and then sampling $a\sim p(\cdot|x)$.

The true conditional expectation function of rewards is denoted by $f^*:\Xscript \times \A \rightarrow [0,1]$; i.e. $f^*(x,a):=\E[r_t(a)|x_t=x]$. We also let $D_{\Xscript}$ denote the marginal distribution of $D$ on the set of contexts $\Xscript$. A model $f \in \F$ is any map from $\Xscript\times\A$ to $[0,1]$. With a slight abuse of notation, for any model $f\in\F$ and context $x\in\Xscript$, we let $f(x)$ denote the vector $(f(x,a))_{a\in\A}$ that lies in $[0,1]^K$.

By a policy we mean is a deterministic function from $\pi : \Xscript \to \A$. The policy that maximizes the conditional mean of rewards is denoted by $\pi^*(x)$; i.e., $\pi^{*}(x) = \arg\max_a f^*(x, a)$. We also let $\pi_f$ denote the policy induced by model $f$, which is given by $\pi_f(x) := \arg\max_a f(x, a)$ for every $x$. The goal of a contextual bandit algorithm is to bound cumulative regret:
\begin{equation}
    \label{eq:cum_regret}
    R_T := \sum_{t=1}^T [r_t(\pi^*(x_t))- r_t(a_t)].
\end{equation}

\paragraph{Misspecification} 
Since our algorithm does not require that the model be well-specified, the class $\F$ may not contain the true model $f^*$. We denote by $\sqrt{B}$ the \emph{average misspecification} error relative to $\F$,\footnote{Similar measures of misspecification are denoted by $\epsilon$ in other papers.} where
\begin{equation}
    \label{eq:capb}
    B :=\max_{p\in\kernelSet} \min_{f\in\F} \E_{x\sim D_{\Xscript}}\E_{a\sim p(\cdot|x)} [(f(x,a)-f^*(x,a))^2].
\end{equation}
Where $\kernelSet$ is the set of probability kernels induced by $\F$,
\begin{equation}
    \label{eq:action-kernel-set}
    \begin{aligned}
    &\kernelSet:=\bigg\{p:\A\times\Xscript\rightarrow [0,1] \text{ is a probability kernel} |\;\\ 
    &\exists f_p\in\F, \; \text{probability kernel } g_p:\A\times [0,1]^K \rightarrow [0,1], \;\\ 
    &\text{such that } \forall \; (x,a)\in\Xscript \times \A, \; p(a|x) = g_p(a| f_p(x)) \bigg\}.
    \end{aligned}
\end{equation}
The set $\kernelSet$ contains all probability kernels $p$ from $\A\times\Xscript$ to $[0,1]$ that can be represented by some pair $(f_p,g_p)$, where $f_p$ is a function in $\F$ and $g_p$ is a probability kernel from $\A\times [0,1]^K$ to $[0,1]$ such that $p(a|x)=g_p(a|f_p(x))$ for all actions $a$ and contexts $x$. That is, $\kernelSet$ considers only those probability kernels that depend on the context $x$ through some some model in $\F$. We discuss this in more detail in \Cref{sec:lower-bound}.

The average misspecification need not be known, though our regret bounds stated in Section \ref{sec:main_results} will depend on it.

\subsection{Regression Oracle}
\label{sec:regression-oracle}

We use a regression algorithm as a subroutine on the class of outcome models $\F$, and our exploration depends on the estimation rates of this subroutine. Suppose $\hatf$ is the output of the regression algorithm fitted on $n$ independently and identically drawn samples from $D(p)$, it is then reasonable to expect that for any $\zeta\in(0,1)$, the following holds with probability $1-\zeta$:
\begin{equation}
\label{eq:regression-rate}
\begin{aligned}
    &\E_{x\sim D_{\Xscript}}\E_{a\sim p(\cdot|x)}[ (\hatf(x, a) - f^*(x,a))^2 ]\\ 
    &\leq \min_{f\in\F} \E_{x\sim D_{\Xscript}}\E_{a\sim p(\cdot|x)}[(f(x,a)-f^*(x,a))^2] + \xi(n,\zeta).
\end{aligned}
\end{equation}
We call $\xi(\cdot,\cdot)$ the estimation rate of the regression algorithm and assume that it is known. We also require it to satisfy two benign conditions, and say it is a ``valid'' estimation rate if it satisfies these conditions.  First, we require $\xi$ to be a decreasing function of $n$. In particular, we require:\footnote{We require the first condition to ensure that $\gamma_m$ is a increasing function of $m$, see \Cref{sec:algorithm}.}
\begin{equation}
\label{eq:xi-first-condition}
    \begin{aligned}
        &\text{For all $\delta\in(0,1)$,}\\
        &\text{$\xi(n, \delta/\ln(n))$ is non-increasing in $n$.}
    \end{aligned}
\end{equation}
The second condition is that this estimation rate is lower bounded by the rate for estimating the mean of a one-dimensional bounded random variable:\footnote{The second condition is more for notational convenience as \eqref{eq:regression-rate} will always hold with a larger $\xi$. Further, in most scenarios, one would not expect a rate smaller than the one for estimating the mean of a one-dimensional bounded random variable. }
\begin{equation}
\label{eq:xi-second-condition}
    \begin{aligned}
        &\text{For all $\zeta\in(0,1)$ and $n\in\N$,}\\ 
        &\xi(n, \zeta) \geq \ln(1/\zeta)/n.
    \end{aligned}
\end{equation}
We restate these general requirements of the regression algorithm as \Cref{ass:main-assumption} in \Cref{sec:main_results}. Finally, for concreteness, note that any estimation rate of the following form is valid: 
\begin{equation}
\label{eq:common-rate}
    \xi(n,\zeta) =
    \begin{cases}
        \frac{C\ln^{\rho'}(n)\ln(1/\zeta)\comp(\F)}{n^{\rho}}, &\text{for $n\geq n_0$.} \\
        1. &\text{ otherwise.}
    \end{cases}
 \end{equation}
where $C>0$, $\rho\in (0,1]$, $\rho'\in [0,\infty)$, $\comp(\F)$ is an appropriate measure of the complexity of the outcome model class $\F$, and $n_0 \in \N$ is an appropriately chosen constant that ensures \eqref{eq:xi-first-condition} holds. Many statistical rates have this form \citep[see e.g.,][]{koltchinskii2011oracle}, indicating that our conditions on the regression algorithm are relatively benign.

\subsection{Algorithm}
\label{sec:algorithm}

In this section we outline the $\SafeFalcon$ algorithm. A formal description is deferred to Algorithm \ref{alg:safe-falcon} below.

As the name suggests, our method is based on the \texttt{FALCON+} algorithm in \citep[][Algorithm 2]{simchi2020bypassing}. \texttt{FALCON+} is computationally tractable and, when the model is well-specified (i.e., when $f^* \in \mathcal{F}$), attains optimal regret bounds on cumulative regret \eqref{eq:cum_regret}. However, as we saw in the introduction, under misspecification its behavior can be erratic. 

 $\SafeFalcon$ is implemented in epochs indexed by $m$. Where epoch $m$ starts at round $\tau_{m-1}+1$ and ends at round $\tau_m$, $\tau_{m+1}=2\tau_m$ for all $m\geq 1$, $\tau_0=0$, and $\tau_1\geq 2$ is an input to the algorithm. Each round $t$ starts with a status that is called ``safe'' or ``not safe'', depending on whether the algorithm has detected evidence of model misspecification using a test that will be described shortly. The algorithm's behavior depends on this status, and once it switches to ``not safe'' status, it never returns to ``safe''. Let's describe each of these behaviors.

\paragraph{Status-dependent behavior}  At the beginning of any epoch $m$ that starts on a ``safe'' round, the algorithm uses an estimate of the reward model $\hatf_m$, obtained using data from epoch $m-1$ from an offline regression oracle. As long as the ``safe'' status is maintained, at each round in this epoch it assigns actions by drawing from the following distribution, named the \emph{action selection kernel}: 
 \begin{align}
   \label{eq:action_kernel}
   p_m(a|x):=
   \begin{cases}
    \frac{1}{K+\gamma_m \left(\hatf_m(x, \hat{a}) - \hatf_m(x,a) \right)} &\text{for } a\neq \hat{a},\\
   1 - \sum_{a'\neq \hat{a}} p(a'|x) &\text{for } a=\hat{a}.
   \end{cases}
 \end{align}
where $\hat{a} = \max_a \hatf_m(a, x)$ is the predicted best action. The parameter $\gamma_m > 0$ governs how much the algorithm exploits and explores: assignment probabilities concentrate on the predicted best policy $\hat{a}$ when $\gamma_m$ is large, and are more spread out when $\gamma_m$ is small. During ``safe'' epochs, the speed at which $\gamma_m$ increases is inversely proportional to the square-root of the estimation rate of the regression algorithm:
\begin{align}
    \label{eq:gamma}
    \gamma_m := \sqrt{1/8} \sqrt{K/\xi(\tau_{m-1} - \tau_{m-2}, \delta'/m^2)}
\end{align}
where $\xi$ is the estimation rate of the regression oracle defined in \eqref{eq:regression-rate}, $\delta > 0$ is a confidence parameter, $\delta'=C\delta$ for a universal constant $C>0$, the quantity $\tau_{m-1} - \tau_{m-2}$ is the size of the previous batch, which was used to estimate the model $\hatf_m$. Definition \eqref{eq:gamma} implies that small classes such as linear models allow for a quickly increasing $\gamma_m$ and therefore more exploitation, while large classes require more exploration and therefore $\gamma_m$ increases more slowly. Finally, we let $S_m$ denote the data collected in epoch $m$, and note that $\hatf_{m+1}$ is the output of the regression oracle with $S_m$ as input.

Now suppose misspecification is detected at round $t$ and the algorithm enters ``not safe'' status. For all epochs $m< m(t)$, we compute a high-probability lower bound $l_m'$ around the expected reward of the policy that selects actions according to the kernel $p_m$, and select the action selection kernel associated with the epoch corresponding to the highest lower bound $\hat{m} = \arg\max_{m\leq m(t)} l_m'$, where
\begin{equation}
    l_{m}' := \frac{1}{|S_{m}|}\sum_{(x,a,r)\in S_{m}} r - \sqrt{\frac{1}{2|S_{m}|}\ln\bigg(\frac{m^2}{\delta'}\bigg)}.
\end{equation}
Thereafter, all actions will be selected according the the action selection kernel $p_{\hat{m}}(x|a)$.

\paragraph{Testing for misspecification} Denoting the beginning of the $m$-th epoch by $\tau_{m-1}+1$, the algorithm tests for misspecification at the end of round $\tau_{m-1} + 1$, $\tau_{m-1} + 2$, $\tau_{m-1} + 4$, and so on, up to and including $\tau_{m}$. Suppose round $t$ is one of these time-steps in epoch $m$ where the algorithm tests for misspecification. The test starts by constructing a loose high-probability lower bound on the expected reward of the optimal policy, $l_{m-1}=\max_{m'\leq m-1} l_{m'}'$.\footnote{By construction, $l'_m$ is a high probability lower bound on the expected reward of the randomized policy used in epoch $m$. Hence, $l_m=\max_{m'\leq m} l_{m'}'$ is a high probability lower bound on the expected reward of some (possibly randomized) policy. Therefore, $l_m$ is also a high probability lower bound on the expected reward of the optimal policy. } The test consists of checking whether cumulative rewards $\smash{\sum_{i=1}^t r_i(a_i)}$ remain above some lower bound $L_t$, defined as
\begin{equation}
    \label{eq:lower_bound}
    \begin{aligned}
     L_t := &t\cdot l_{m-1} - \tau_1 - \sqrt{2t\ln\Big(\frac{\lceil m+\log_2(\tau_1) \rceil^3}{\delta'}\Big)}\\ 
     & - 20.3\sqrt{K}\sum_{i=\tau_2}^{t} \sqrt{\xi\Big(\tau_{m(i)-1}-\tau_{m(i)-2}, \frac{\delta'}{(m(i))^2}\Big)}.
     \end{aligned}
\end{equation}
Once we detect that cumulative reward dip below this bound, the algorithm switches to ``not safe'' status forever. 

\begin{algorithm}[h]
  \caption{$\SafeFalcon$}
  \label{alg:safe-falcon}
  \textbf{input:} Initial epoch length $\tau_1\geq 2$, confidence parameter $\delta\in(0,1)$.
  \begin{algorithmic}[1] 
  \STATE Set $\tau_0 = 0$, and $\tau_{m+1} = 2\tau_m$ for all $m\geq 1$.
  \STATE Let $\hatf_1 \equiv 0$, $l_0 = 0$, $\CumulativeReward_0=0$ , $\safe=\true$, and $\msafealg = 0$.
  \FOR{epoch $m=1,2,\dots$}
    \STATE Let $\gamma_m$ be given by \eqref{eq:gamma}. (for epoch 1, $\gamma_1=1$.)\\
    \FOR{round $t=\tau_{m-1}+1,\dots, \tau_{m}$ }
        \STATE Observe context $x_t$.
        \IF{$\safe$}
            \STATE Let $p_m$ be given by \eqref{eq:action_kernel}.
            \STATE Sample $a_t \sim p_m(\cdot|x_t)$, and observe $r_t(a_t)$.
            \STATE Let $\CumulativeReward_t\leftarrow \CumulativeReward_{t-1}+r_t(a_t)$.
            \IF{$m\geq 2$} 
            \STATE $\safe \leftarrow \IsSafe(m, t, l_{m-1}, \CumulativeReward_t)$.
            \ENDIF
        \ELSE
            \STATE Sample $a_t\sim p_{\msafealg}(\cdot|x_t)$.
        \ENDIF
    \ENDFOR
    \IF{$\safe$}
        \STATE Let $S_m:=\{(x_t,a_t,r_t(a_t))\}_{t=\tau_{m-1}+1}^{\tau_m}$, be the data collected in epoch $m$.
        \STATE Let $(l_m,\msafealg)\leftarrow \ChoosePi(m, S_m, l_{m-1})$.
        \STATE Let $\hatf_{m+1}$ be the output of the regression algorithm with $S_{m}$ as input.
    \ENDIF
  \ENDFOR
  \end{algorithmic}
\end{algorithm}

\begin{algorithm}[h]
  \caption{$\IsSafe$}
  \label{alg:check-is-safe}
  \textbf{input:} Epoch $m$, time-step $t$, lower bound $l_{m-1}$, and $\CumulativeReward_t$.
  \begin{algorithmic}[1] 
  \STATE Let $L_t$ be given by \eqref{eq:lower_bound}.
  \IF{$\log_2(t-\tau_{m-1})\in\{0,1,2,\cdots\}$ or $t=\tau_m$}
    \IF{$\CumulativeReward_t\geq L_t$} 
    \STATE $\safe \leftarrow \true$.
    \ELSE
    \STATE $\safe \leftarrow \textbf{False}.$
    \ENDIF
  \ENDIF
  \STATE Return $\safe$.
  \end{algorithmic}
\end{algorithm}

\begin{algorithm}[h]
  \caption{$\ChoosePi$}
  \label{alg:choose-safe}
  \textbf{input:} Epoch $m$, lower bound $l_{m-1}$, and data collected in the $m$-th epoch $S_m$.
  \begin{algorithmic}[1] 
  \STATE Let
  \begin{align*}
      l_m' = \frac{1}{|S_{m}|}\sum_{(x,a,r)\in S_{m}} r \; - \sqrt{\frac{1}{2|S_{m}|}\ln\bigg(\frac{m^2}{\delta}\bigg)}.
  \end{align*}
  \STATE Let $l_m = \max(l_{m-1},l_m')$.
  \IF{$l_m\neq l_{m-1}$}
    \STATE Update $\msafealg\leftarrow m$.
  \ENDIF
  \STATE Return $(l_m,\msafealg)$.
  \end{algorithmic}
\end{algorithm}

\subsection{Understanding Safe-FALCON}
\label{sec:understanding-falcon}

In this section we try to understand $\SafeFalcon$ and simultaneously sketch a proof for our main cumulative regret bound (\Cref{thm:main-theorem}). \Cref{thm:main-theorem} provides the following bound on cumulative regret: 
\begin{align}
\label{eq:main-result-copy}
\begin{split}
   R_T \leq \ordO\Bigg(  &\sqrt{KB}\;T \\ 
   + \sqrt{K}&\sum_{t=\tau_1+1}^{T} \sqrt{\xi\Big(\tau_{m(t)-1}-\tau_{m(t)-2}, \frac{\delta}{(m(t))^2}\Big)} \Bigg). 
\end{split}
\end{align}
The first term in \eqref{eq:main-result-copy} is the regret overhead due to misspecification, and the second term is the regret bound for \texttt{FALCON+} assuming realizability holds. We also briefly note that all expectations in this section are taken over the randomness in the environment and algorithm being used.

To understand the intuition behind the algorithm, consider the following epoch, 
\begin{equation}
\label{eq:msafe}
    \msafe := \max \Bigg\{ m \suchthat B \leq \xi\Big(\tau_m-\tau_{m-1}, \frac{\delta'}{m^2} \Big) \Bigg\}.
\end{equation}
We show that, with high probability, the status at the end of epoch $\msafe+1$ is ``safe''. Moreover, up to the end of epoch $\msafe+1$, our upper bound on the expected instantaneous regret decreases at the same rate as when realizability holds. The proof of this fact follows by making a ``simple'' observation that allows us to extend the  analysis of \cite{simchi2020bypassing} to bound cumulative regret in these early epochs. In particular, if $m(t)\leq \msafe+1$, we get that with high probability the expected cumulative regret up to time $t$ is upper bounded by:
\begin{equation}
\label{eq:expected-cum-regret-upper bound}
    \begin{aligned}
    &\E\bigg[\sum_{i=1}^{t} (r_i(\pi^*(x_i))- r_i(a_i))\bigg]\\ 
    &\leq \tau_1 + 20.3\sqrt{K}\sum_{i=\tau_2}^{t} \sqrt{\xi\Big(\tau_{m(i)-1}-\tau_{m(i)-2}, \frac{\delta'}{(m(i))^2}\Big)},
    \end{aligned}
\end{equation}
which are the bounded attained by \texttt{FALCON+} under realizability. After epoch $\msafe+1$, the expected instantaneous regret may increase. However, we show that the lower bound we construct for the expected reward of the policy that selects according to the action selection kernel $p_{\msafe+1}$ is sufficiently close to the expected reward of the optimal policy:
\begin{equation}
\label{eq:bound-l-prime}
    \E[r_t(\pi^*(x_t))] - l_{\msafe+1}' \leq  \ordO(\sqrt{KB}).
\end{equation}
Hence, if we knew $\msafe$, by switching the algorithm's status to ``not safe'' at the end of epoch $\msafe+1$ would give us the required bounds on cumulative regret (\Cref{thm:main-theorem}). Unfortunately,
we do not know the value of $\msafe$. So, we try to detect if our current epoch $m$ is larger than $\msafe+1$ by looking for unexpected jumps in cumulative regret.\footnote{Note that in the realizable case, $\msafe$ is $\infty$, therefore trying to detect if $m>\msafe+1$ can also be considered as a test for misspecification as it is also testing the finiteness of $\msafe$.} 
Recall that from the construction of $l_{m-1}$ described in \Cref{sec:algorithm}, we have have that, $l_{m-1}$ is a weak high-probability lower bound on the expected reward of the optimal policy ($\E[r_t(\pi^*(x_t))]$). Therefore, from \eqref{eq:expected-cum-regret-upper bound} we get that when $m(t)\leq \msafe+1$, the expected cumulative reward up to time $t$ should be lower bounded by: 
\begin{equation}
\label{eq:expected-cum-reward-lower-bound}
    \begin{aligned}
     &\E\bigg[\sum_{i=1}^t r_i(a_i)\bigg] \geq t\cdot l_{m-1} - \tau_1 \\ 
     &- 20.3\sqrt{K}\sum_{i=\tau_2}^{t} \sqrt{\xi\Big(\tau_{m(i)-1}-\tau_{m(i)-2}, \frac{\delta'}{(m(i))^2}\Big)}.
    \end{aligned}
\end{equation}
Now, from standard concentration arguments and \eqref{eq:expected-cum-reward-lower-bound}, we get that with high probability, the cumulative reward up to time $t$ must be lower bounded by $L_t$ if $m(t)\leq \msafe+1$. That is, with high-probability, our test claims that $m(t)>\msafe+1$ only if it is true. This completes the proof sketch for the validity of the misspecification test. Hence, by design, we get that the status at the end of epoch $\msafe+1$ is safe with high probability.

Finally consider the case when $m(t)>\msafe+1$, but the misspecification test was not violated. That is $m(t)>\msafe+1$, but the cumulative reward up to time $t$ is lower bounded by $L_t$. By our algorithm design, since $m(t)>\msafe+1$, we get that: 
\begin{equation}
    \label{eq:order-l}
    l_{m(t)-1} \geq l_{\msafe+1} \geq l'_{\msafe+1}.
\end{equation}
Combining \eqref{eq:bound-l-prime} and \eqref{eq:order-l}, gives us that with high-probability, the lower bound $l_{m(t)-1}$ is close to the expected reward of the optimal policy:
\begin{equation}
    \label{eq:bound-lmt-1}
    \E[r_t(\pi^*(x_t))] - l_{m(t)-1} \leq  \ordO(\sqrt{KB}).
\end{equation}
Combining \eqref{eq:bound-lmt-1} and the fact that cumulative reward up to time $t$ is lower bounded by $L_t$, we get the required bound on cumulative regret \eqref{eq:main-result-copy}.\footnote{We use \eqref{eq:bound-lmt-1} to lower bound $L_t$ in terms of the expected optimal reward. We then use standard concentration inequalities to further lower bound this in terms of the cumulative reward of the optimal policy. Since $L_t$ itself is a lower bound on the cumulative reward up to time $t$, we get the required bound on cumulative regret \eqref{eq:main-result-copy}. } 

As we argued earlier, with high-probability, the algorithm's status switch only happens after epoch $\msafe+1$. From \eqref{eq:bound-l-prime}, we get that the instantaneous regret after the status switch is sufficiently small to give us the required bound on the cumulative regret \eqref{eq:main-result-copy}. This completes the proof sketch for \Cref{thm:main-theorem} and also explains our algorithmic choices in $\SafeFalcon$.

\subsection{Main result}
\label{sec:main_results}

The performance of our algorithm will depend on known estimation rates of the regression algorithm. As discussed in \Cref{sec:regression-oracle}, we require the regression algorithm used in $\SafeFalcon$ to satisfy \Cref{ass:main-assumption} described below.

\begin{assumption}
\label{ass:main-assumption}
Suppose that the regression algorithm used on the class of outcome model $\F$ satisfies the following property. For any probability kernel $p\in \kernelSet$, any natural number $n$, and any $\zeta \in (0, 1)$, the following holds with probability at least $1-\zeta$:
\begin{equation}
\label{eq:assumption}
    \E_{x\sim D_{\Xscript}}\E_{a\sim p(\cdot|x)}[ (\hatf(x, a) - f^*(x,a))^2 ] \leq B + \xi(n,\zeta).
\end{equation}
and where $\hatf$ is the output of the regression algorithm fitted on $n$ independently and identically drawn samples from $D(p)$ as input. Here $B>0$ is a (possibly unknown) constant. The function $\xi:\N\times[0,1]\rightarrow [0,\infty)$ is a known, ``valid'' rate; i.e., it satisfies \eqref{eq:xi-first-condition} and \eqref{eq:xi-second-condition}.\footnote{For regression algorithms that satisfy \eqref{eq:regression-rate}, we get that the constant $B$ used in \cref{ass:main-assumption} is given by \eqref{eq:capb}.}
\end{assumption}

\begin{restatable}[Main result]{theorem}{thmMain}
\label{thm:main-theorem} Suppose the regression algorithm used in $\SafeFalcon$ satisfies \Cref{ass:main-assumption}. Then with probability at least $1 - \delta$, $\SafeFalcon$ attains the following regret guarantee:
\begin{align}
\label{eq:main-theorem}
\begin{split}
   R_T \leq \ordO\Bigg(  &\sqrt{KB}\;T \\ 
   + \sqrt{K}&\sum_{t=\tau_1+1}^{T} \sqrt{\xi\Big(\tau_{m(t)-1}-\tau_{m(t)-2}, \frac{\delta}{(m(t))^2}\Big)} \Bigg). 
\end{split}
\end{align}
\end{restatable}
The above regret typically has the same rate as $\ordO(\sqrt{K\xi(T,\delta/\log(T))}T + \sqrt{KB}T)$. In particular, when the estimation rates in \cref{ass:main-assumption} are of the form \eqref{eq:common-rate}, we get the regret bound given by \Cref{cor:common-rate}.

\begin{restatable}[]{corollary}{corCommonrate}
\label{cor:common-rate} 
Suppose the regression algorithm used in $\SafeFalcon$ satisfies \Cref{ass:main-assumption} with estimation rate of the form given by \eqref{eq:common-rate}. Then with probability at least $1 - \delta$, $\SafeFalcon$ attains the following regret guarantee:
\begin{align}
\label{eq:cor-common-rate}
\begin{split}
   R_T \leq \ordO\Bigg( & \sqrt{KB}\;T\\ 
   + &\sqrt{KT^{2-\rho}\ln^{\rho'}(T)\ln\Big(\frac{\ln(T)}{\delta}\Big)\comp(\F)} \Bigg). 
\end{split}
\end{align}
\end{restatable}

Note that \Cref{thm:main-theorem} provides a bias-variance trade-off for contextual bandits. The first term in \eqref{eq:main-theorem} (regret overhead due to misspecification) depends on $B$, which is a tight upper bound on the average squared bias for the best estimator in the model class $\F$ under the distribution induced by any probability kernel in $\kernelSet$. The second term in \eqref{eq:main-theorem} (regret bound under realizability) depends on the estimation rate $\xi(\cdot,\cdot)$, which captures the variance of the regression oracle estimate over the class $\F$. For more expressive model classes, the bias term $B$ is small, but the variance term $\xi(\cdot,\cdot)$ is large, showing that there is a bias-variance trade-off for contextual bandits that rely on some model class $\F$. A better dependency on the variance term cannot be expected even when realizability holds \citep[see e.g.][]{foster2020beyond}. In \Cref{thm:lower-bound}, we show that one cannot get a better dependency on the bias term either by providing a $\Omega(\sqrt{KB})$ lower bound on the regret overhead due to misspecification for contextual bandits that use regression oracles or rely on a model class $\F$. 


\subsection{Lower bound}
\label{sec:lower-bound}

We prove a new lower bound on the regret overhead due to misspecification for the stochastic contextual bandit setting in terms of the average misspecification error $\sqrt{B}$, where $B$ is defined in \eqref{eq:capb}. 

The issue of model misspecification is specific to contextual bandit algorithms that use regression oracles or rely on some model class $\F$. To state a lower bound on the regret overhead due to misspecification, it is helpful to understand the common characteristics of such algorithms. We argue that the set $\kernelSet$ is central to many contextual bandit algorithms based on regression oracles. In particular, we will argue that at every time-step $t$ such algorithms choose some probability kernel $\tilde{p}_t$ in the convex hull of $\kernelSet$, receive a context $x_t$, and sample an action $\tilde{a}_t$ from $\tilde{p}_t(\cdot|x_t)$.



For example, at every time-step, the algorithms used in \cite{foster2020beyond}, \cite{simchi2020bypassing}, and this work use probability kernels of the form defined in \eqref{eq:action_kernel}, which are in $\kernelSet$.
Similarly, parametric Thompson Sampling algorithms \citep[e.g.,][]{agrawal2013thompson} select actions by following probability kernels that lie in the convex hull of $\kernelSet$. This is because, in our notation,  Thompson Sampling algorithms at every time-step sample a function $\tilde{f}$ from the class $\F$ and then follow the policy $\pi_{\tilde{f}}$, which corresponds to some kernel in $\kernelSet$. The same is true for greedy and epsilon-greedy algorithms that select actions based on a regression oracle since uniform sampling does not depend on contexts and since the greedy policy $\pi_f$ corresponds to some probability kernel in $\kernelSet$.


While algorithms based on upper confidence bounds may not use policies that correspond to kernels in the convex hull of $\kernelSet$, we informally note that these algorithms are asymptotically greedy and hence converge to policies that correspond to kernels in $\kernelSet$.\footnote{UCB algorithms rely on confidence estimates. It wasn't clear to us what the general form of these confidence estimates should be and how they would relate to $\F$.}


\Cref{thm:lower-bound} shows that there is a family of stochastic contextual bandit instances such that, for any probability kernel in the convex hull of $\kernelSet$, the expected instantaneous regret of the induced randomized policy can be lower bounded by $\Omega(\sqrt{KB})$. Hence on these instances, any algorithm that plays randomized policies induced by probability kernels in the convex hull of $\kernelSet$ for at least a constant fraction of time-steps has expected cumulative regret lower bounded by $\Omega(\sqrt{KB}\cdot T)$.

\begin{restatable}[Lower bound]{theorem}{thmLbound}
\label{thm:lower-bound}
Consider any $K\geq 2$ and $B\in [0,1/(2K)]$. One can construct a model class $\F$ and a stochastic contextual bandit instance with $K$ arms. Such that the average misspecification error is $\sqrt{B}$. And for any probability kernel $p$ in the convex hull of the kernel set $\kernelSet$, the expected instantaneous regret of the induced randomized policy can be lower bounded by:
\begin{align}
\label{eq:instantaneous-lower-bound-thm}
\begin{split}
    \E_{(x,r)\sim D} \E_{a\sim p(\cdot|x)} \big[r(\pi^*(x))-r(a)\big] \geq \Omega(\sqrt{KB}) 
\end{split}
\end{align}
\end{restatable}

An immediate implication of \Cref{thm:lower-bound} is that the regret overhead due to misspecification for most contextual bandit algorithms that use regression oracles can be lower bounded by $\Omega(\sqrt{KB}\cdot T)$. Hence showing that the regret upper bound (\Cref{thm:main-theorem}) ensured by $\SafeFalcon$ is optimal. 

\subsection{Improving $\SafeFalcon$ and a Simulation}

In \Cref{sec:algorithm}, we discussed a misspecification test ($\IsSafe$) that checks if the cumulative reward remains above a lower bound $L_t$. At every round where we verify this condition, we can similarly check if the average per-epoch reward remains above a lower bound (see \eqref{eq:average-epoch-reward-lower-bound}). Similar to the argument used in \Cref{sec:understanding-falcon}, one can show that \eqref{eq:average-epoch-reward-lower-bound} holds with high-probability if $m(t)\leq \msafe+1$. Hence adding this test to $\IsSafe$ can only make $\SafeFalcon$ more robust, ensuring \Cref{thm:main-theorem} continues to hold.  
\begin{equation}
\label{eq:average-epoch-reward-lower-bound}
    \begin{aligned}
     &\frac{1}{t-\tau_{m(t)-1}}\sum_{i=\tau_{m(t)-1}}^t r_i(a_i) \geq  l_{m(t)-1} \\
     &- 20.3\sqrt{K}\sqrt{\xi\Big(\tau_{m(t)-1}-\tau_{m(t)-2}, \frac{\delta'}{(m(t))^2}\Big)}\\
      & - \sqrt{\frac{2}{t-\tau_{m(t)-1}}\ln\bigg(\frac{\lceil m(t)+\log_2(\tau_1)\rceil^3}{\delta'}\bigg)}
    \end{aligned}
\end{equation}

Further improvements to $\IsSafe$ can be made by constructing better high-probability lower bounds ($l_{m-1}$) on the expected reward of the optimal policy. One approach to constructing such bounds would be to use offline policy evaluation methods to construct a lower bound on the expected reward of a policy that is estimated to be optimal. We do not pursue this here.

\begin{figure}[h]
    \centering
    \includegraphics[width=0.44\textwidth]{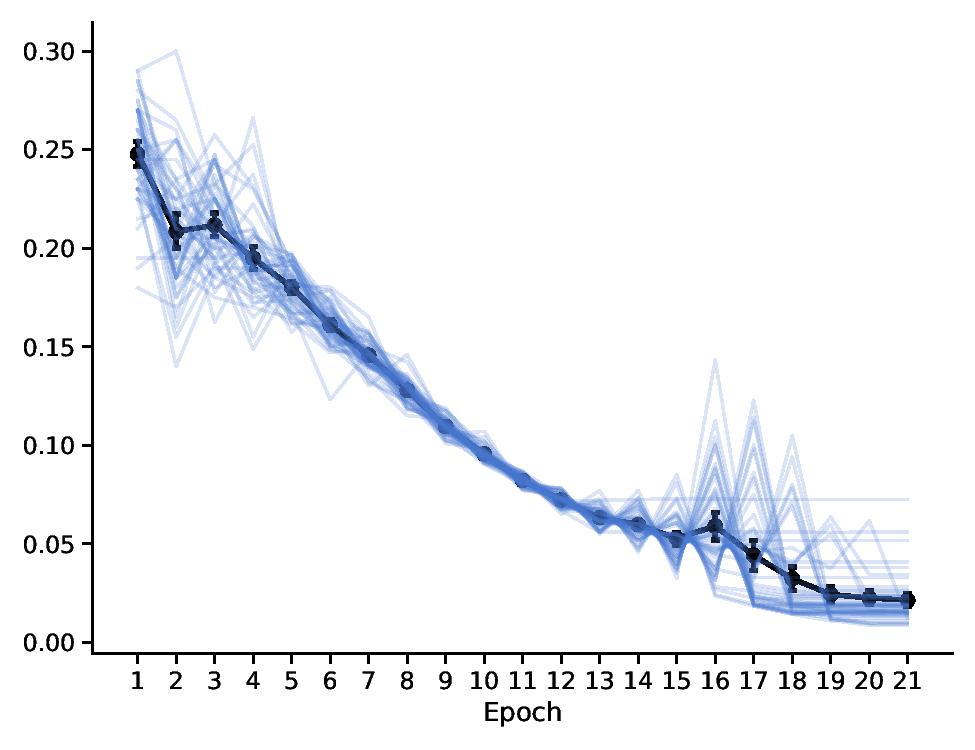}
    \caption{Illustrating that linear $\SafeFalcon$ does not fail on Example \eqref{eq:intro_example_model}. Each epoch starts at round $2^m$. Vertical bars are 95\% confidence intervals around the average per-epoch average regret, aggregated over 50 simulations.}
    \label{fig:safe-falcon}
\end{figure}

To complete our discussion from \Cref{sec:introduction}, we simulate a version of linear $\SafeFalcon$ on Example~\eqref{eq:intro_example_model}. In particular, we implement a version of $\SafeFalcon$ that uses two misspecification tests, a test that checks if the cumulative reward remains above a lower bound (line 3 of $\IsSafe$) and a test that checks if the average per-epoch reward remains above a lower bound \eqref{eq:average-epoch-reward-lower-bound}. Other parameters are chosen as in the introduction example (see \Cref{sec:introduction_example} for details). The results are shown in Figure \ref{fig:safe-falcon}. 



Despite this example not being linearly realizable, in contrast to \texttt{FALCON+} (see \Cref{fig:falcon_failure}), average per-epoch regret under $\SafeFalcon$ does not increase in later periods (see \Cref{fig:safe-falcon}). $\SafeFalcon$ detects misspecification sometime after epoch 12 and defaults to the action selection kernel used in epoch $\hat{m}$ thereafter. For each simulation, the selected safe policy is fixed and attains constant regret, which explains the horizontal lines seen on the right side of the graph in \Cref{fig:safe-falcon}. An interesting direction for future improvement is to develop algorithms that continue adaptive experimentation after epoch $\msafealg$.



\section{Discussion}

In this work, we presented a contextual bandit algorithm that is computationally tractable, flexible, and supports general-purpose function approximation. The ideas used here are relatively simple and allow us to provide a reduction from contextual bandits to offline regression without assuming realizability. We do this by modifying the \texttt{FALCON+} algorithm, allowing us to inherit the optimal guarantees of \cite{simchi2020bypassing} when realizability holds. When realizability doesn't hold, we get an optimal bound on the regret overhead due to misspecification in terms of the average misspecification error. We provide both upper (\Cref{thm:main-theorem}) and lower (\Cref{thm:lower-bound}) bounds on regret, allowing us to quantify the bias-variance trade-off for contextual bandit algorithms based on regression oracles.



\section{Acknowledgments}

We are grateful for the generous financial support provided by the Sloan Foundation, Schmidt Futures and the Office of Naval Research grant N00014-19-1-2468. SKK acknowledges generous support from the Dantzig-Lieberman Operations Research Fellowship.

\bibliography{example_paper}
\bibliographystyle{icml2021.bst}

\clearpage
\appendix
\onecolumn

\section{Outline}
\label{app:outline}


In \Cref{app:preliminaries}, we establish additional notation that will useful for our proofs. We detail the proofs for the upper and lower bounds in \Cref{app:proof-upper} and \Cref{app:proof-lower} respectively. Finally in \Cref{sec:introduction_example} we explain the introductory example in more detail and provide some more intuition.

\subsection{Preliminaries}
\label{app:preliminaries}

Most of the notation and definitions described below will be the same as in \cite{simchi2020bypassing}. 

A ``policy'' is a deterministic mapping from contexts to actions. Let $\Psi = \A^{\Xscript}$ be the universal policy space containing all possible policies. The expected instantaneous reward of the policy $\pi$ with respect to the model $f$ is defined as
\begin{equation}
    \label{eq:instantaneous-reward}
    R_f(\pi) := \E_{x \sim D_\Xscript}[f(x, \pi(x))].
\end{equation}

Recalling that $f^*$ is the true conditional means of rewards, we write $R(\pi)$ to mean $R_{f^*}(\pi)$, the true expected instantaneous reward for policy $\pi$. The policy $\pi_f$ that is induced by model $f$ is given by $\pi_f(x) := \arg\max_a f(x, a)$ for every $x$. This policy has the highest instantaneous reward with respect to the model $f$, that is $\pi_f=\arg\max_{\pi\in\Psi} R_f(\pi)$.  

The expected instantaneous regret of a policy $\pi$ with respect to the outcome model $f$ is defined as
\begin{equation}
    \label{eq:instantaneous-regret}
    \Reg_f(\pi) := \E_{x \sim D_\Xscript}[f(x,\pi_f(x)) - f(x, \pi(x))].
\end{equation}
 
We write $\Reg(\pi)$ to mean $\Reg_{f^*}(\pi)$, the true expected instantaneous regret for policy $\pi$. We also let $\Gamma_{t}$ denote the set of observations up to and including time $t$. That is
\begin{align}
    \label{eq:history}
    \Gamma_{t} := \{(x_s, a_s, r_s(a_s))\}_{s=1}^{t}
\end{align}

Given any probability kernel $p$ from $\A\times\Xscript$ to $[0,1]$, from Lemma 3 in \cite{simchi2020bypassing}, there exists a unique product probability measure on $\Psi$, given by:
\begin{equation}
    \label{eq:q-product}
    Q_p(\pi) := \prod_{x\in\Xscript} p(\pi(x)|x).
\end{equation}
This measure satisfies the following property
\begin{align}
    \label{eq:connecting-p-and-Q}
  p(a|x) = \sum_{\pi\in\Psi} \I\{\pi(x)=a\}Q_p(\pi).
\end{align}
Since any probability kernel $p$ from $\A\times\Xscript$ to $[0,1]$ induces the distribution $Q_p$ over the set of deterministic policies $\Psi$, we can think of $Q_p$ as a randomized policy induced by $p$. Equations \eqref{eq:connecting-p-and-Q} and \eqref{eq:q-product} establish a correspondence between the probability kernel $p$ and the induced randomized policy $Q_p$. For any probability kernel $p$ and any policy $\pi$, we let $V(p,\pi)$ denote the expected inverse probability. \footnote{In \cite{simchi2020bypassing}, this term is called the decisional divergence between the randomized policy $Q_p$ and deterministic policy $\pi$.}
\begin{equation}
    \label{eq:decisional-divergence}
    V(p,\pi):=\E_{x\sim D_{\Xscript}}\bigg[\frac{1}{p(\pi(x)|x)}\bigg]
\end{equation}

\section{Proof for upper bound}
\label{app:proof-upper}

In this section we prove \Cref{thm:main-theorem}, we start by establishing some more additional notation. We say an epoch $m$ is safe when the status at the end of the epoch is safe, that is the variable $\safe$ is still set to $\true$ at the end of this epoch. Let $\msafealg$ be the last safe epoch, that is:
\begin{equation}
    \label{eq:msafealg}
    \begin{aligned}
      \msafealg := \max \Bigg\{ m \suchthat \text{ the epoch $m$ is safe} \Bigg\}.
    \end{aligned}
\end{equation}
Note that, for all $m\leq \msafealg$, the epoch $m$ is safe. Now let $\msafe$ be such that:
\begin{equation}
    \label{eq:msafe-2}
    \begin{aligned}
      \msafe := \max \Bigg\{ m \suchthat B \leq \xi\Big(\tau_m-\tau_{m-1}, \frac{\delta'}{m^2} \Big) \Bigg\}.
    \end{aligned}
\end{equation}
Where $\delta'=\delta/13$. As discussed in \Cref{sec:understanding-falcon}, the epoch $\msafe$ is critical to our theoretical analysis and we will show that with high-probability $\msafe+1$ is safe. We also let $\TsafeCheck$ be the set of time-steps where \cref{alg:check-is-safe} checks the safety condition (see $\IsSafe$). 
\begin{equation}
\label{eq:TsafeCheck-def}
    \begin{aligned}
    \TsafeCheck := \bigg\{ t\in[T] | \text{ status is safe at the start of the time step, $t=\tau_{m(t)}$ or $\log_2(t-\tau_{m(t)-1})$ is integral} \bigg\}.
    \end{aligned}
\end{equation}

For short hand, we let $Q_m\equiv Q_{p_m}$. With some abuse of notation, we let $p_t$ denote the action selection kernel used at time-step $t$. Again with some abuse of notation, we let $Q_t\equiv Q_{p_t}$. 

\subsection{High probability events}
\label{sec:high-prob-events-1}

\Cref{thm:main-theorem} provides certain high probability bounds on cumulative regret. As a preliminary step in these proofs, it will be helpful to show that the events $\eventReg$, $\eventEval$, $\eventEvalRegret$ defined below hold with high-probability. Event $\eventReg$ ensures that our regression estimates are "good" models for the first few epochs. Event $\eventEval$ ensures that the lower bounds constructed by $\ChoosePi$ are valid, this event also includes symmetric upper bounds. Event $\eventEvalRegret$ helps us show that the misspecification tests that we use in $\IsSafe$ are valid.

We start with the event $\eventReg$. This describes the event where, for any epoch $m\in [\msafe]\cap [\msafealg]$, the expected squared error difference between the true model ($f^*$) and the estimated model ($\hatf_{m+1}$) can be bounded purely in terms of the known estimation rate of the regression algorithm.
\begin{equation}
    \label{eq:w-event-Reg}
    \begin{aligned}
      \eventReg := \Bigg\{ \forall m \in [m^*]\cap [\msafealg], \; \E_{x\sim D_{\Xscript}}\E_{a\sim p_m(\cdot|x)}[(\hatf_{m+1}(x,a)-f^*(x,a))^2] \leq 2\xi\Big(\tau_m-\tau_{m-1}, \frac{\delta'}{m^2} \Big) \Bigg\}.
    \end{aligned}
\end{equation}

\begin{lemma}
\label{lem:high-prob-eventReg}
Suppose the regression algorithm used in $\SafeFalcon$ satisfies \Cref{ass:main-assumption}. Then the event $\eventReg$ holds with probability at least $1-2\delta'$.
\end{lemma}
\begin{proof}
Consider any epoch $m$ such that epoch $m-1$ was safe. Since epoch $m-1$ was safe, for the first time-step in epoch $m$ the algorithm samples an action from the probability kernel $p_m$. It may so happen that the status of the algorithm switches at the end of some time-step in epoch $m$ and the algorithm no longer samples actions according to the kernel $p_m$. As long as the status of the algorithm does not switch in epoch $m$, for the purposes of estimating $\hatf_{m+1}$, we want to argue that the data collected in epoch $m$ can be treated as iid samples from the distribution $D(p_m)$. 

One way to see this is by considering the following thought experiment. At the start of epoch $m$, the environment generates $\tau_m-\tau_{m-1}$ data points that are iid sampled from the distribution $D(p_m)$. The environment then runs the regression algorithm on this data, and generates the model $\hatf_{m+1}$. The environment sequentially shows us these data points throughout epoch $m$, as long as the status of the algorithm is safe. Additionally, we also observe $\hatf_{m+1}$ at the end of epoch $m$ if the status of the algorithm was safe throughout the epoch. Regardless of whether we observe the model $\hatf_{m+1}$, the environment constructs $\hatf_{m+1}$ by running the regression algorithm on iid samples from $D(p_m)$.

Further note that the kernel $p_m$ lies in $\kernelSet$. Hence from \Cref{ass:main-assumption}, with probability $1-\delta'/m^2$, we have:
\begin{equation}
    \label{eq:mse_event}
    \E_{x\sim D_{\Xscript}}\E_{a\sim p_m(\cdot|x)}[(\hatf_{m+1}(x,a)-f^*(x,a))^2] \leq B + \xi\Big(\tau_m-\tau_{m-1}, \frac{\delta'}{m^2} \Big).
\end{equation}
Therefore, the probability that \eqref{eq:mse_event} does not hold for some epoch $m\in[\msafealg]$ can be bounded by:
$$ \sum_{m=1}^{\infty} \frac{\delta'}{m^2} \leq 2\delta'. $$
Hence from the definition of $\msafe$, we get that $\eventReg$ holds with probability at least $1-2\delta'$.
\end{proof}

\begin{lemma}
\label{lem:conditional-reward}
For any time-step $t \geq 1$, we have:
\begin{align*}
    &\E_{x_t,r_t,a_t}[r_t(a_t)|\Gamma_{t-1}] =  \sum_{\pi\in\Psi}Q_t(\pi)R(\pi).
\end{align*}
\end{lemma}
\begin{proof}
Consider any time-step $t
\geq 1$, then from \cref{eq:connecting-p-and-Q} relating $p(\cdot|\cdot)$ and $Q_p$ we have the following equalities:
\begin{align*}
\begin{split}
  &\E_{x_t,r_t,a_t}[r_t(a_t)|\Gamma_{t-1}]\\
  & = \E_{x\sim D_{\Xscript}, a \sim p_t(\cdot|x)}[f^*(x,a)]\\
  & = \E_{x\sim D_{\Xscript}}\Bigg[\sum_{a\in\A}p_t(a|x)f^*(x,a)\Bigg]\\
  & = \E_{x\sim D_{\Xscript}}\Bigg[\sum_{a\in\A}\sum_{\pi\in\Psi}\I(\pi(x)=a)Q_t(\pi)f^*(x,a)\Bigg]\\
  & = \sum_{\pi\in\Psi}Q_t(\pi)\E_{x\sim D_{\Xscript}}[f^*(x,\pi(x))]\\
  & =  \sum_{\pi\in\Psi}Q_t(\pi)R(\pi).
\end{split}
\end{align*}
\end{proof}

The event $\eventEval$ provides upper and lower bounds on the expected reward of the randomized policy $Q_m$, for all $m\in[\msafealg]$. 
\begin{equation}
    \label{eq:w-event-Eval}
    \begin{aligned}
      \eventEval := \Bigg\{ \forall m \in [\msafealg], \; &\sum_{\pi\in\Psi} Q_{m}(\pi)R(\pi) \geq  \frac{1}{|S_{m}|}\sum_{(x,a,r)\in S_{m}} r \; - \sqrt{\frac{1}{2|S_{m}|}\ln\bigg(\frac{m^2}{\delta'}\bigg)},\\
      & \sum_{\pi\in\Psi} Q_{m}(\pi)R(\pi) \leq  \frac{1}{|S_{m}|}\sum_{(x,a,r)\in S_{m}} r \; + \sqrt{\frac{1}{2|S_{m}|}\ln\bigg(\frac{m^2}{\delta'}\bigg)} \Bigg\}. 
    \end{aligned}
\end{equation}


\begin{lemma}
\label{lem:high-rpob-eventEval}
The event $\eventEval$ holds with probability at least $1-4\delta'$.
\end{lemma}
\begin{proof}
Consider any safe epoch $m$. Similar to the argument in \Cref{lem:high-prob-eventReg}, for the purpose of estimating the expected reward of the randomized policy $Q_m$, we can treat the data generated in epoch $m$ as iid samples from the distribution $D(p_m)$. Since rewards lie in the range $[0,1]$, from Hoeffding's inequality, with probability at least $1-2\delta'/m^2$, we get:
\begin{align*}
    \sum_{\pi\in\Psi} Q_{m}(\pi)R(\pi) \geq  \frac{1}{|S_{m}|}\sum_{(x,a,r)\in S_{m}} r \; - \sqrt{\frac{1}{2|S_{m}|}\ln\bigg(\frac{m^2}{\delta'}\bigg)},\\
    \sum_{\pi\in\Psi} Q_{m}(\pi)R(\pi) \leq  \frac{1}{|S_{m}|}\sum_{(x,a,r)\in S_{m}} r \; + \sqrt{\frac{1}{2|S_{m}|}\ln\bigg(\frac{m^2}{\delta'}\bigg)}.
\end{align*}

Therefore, we get that these bounds hold for all $m\in [\msafealg]$ with probability at least:
$$ 1 - \sum_{m=1}^{\infty} \frac{2\delta'}{m^2} \geq 1 - 4\delta'.$$
\end{proof}

For all time-steps $t'\in\TsafeCheck$, the event $\eventEvalRegret$ provides a lower bound on the cumulative reward up to time $t'$ in terms of the expected cumulative reward.
\begin{equation}
    \label{eq:w-event-Eval-Regret}
    \begin{aligned}
      \eventEvalRegret := \Bigg\{ &\forall t' \in [\TsafeCheck], \; \sum_{t=1}^{t'}\Bigg(\sum_{\pi\in\Psi} Q_{t}(\pi)R(\pi) - r_t(a_t) \Bigg)  \leq  \sqrt{2t'\ln\Bigg(\frac{\lceil m(t')+\log_2(\tau_1)\rceil^3}{\delta'}\Bigg)}   \Bigg\}. 
    \end{aligned}
\end{equation}

\begin{lemma}
The event $\eventEvalRegret$ holds with probability at least $1-\delta'$.
\end{lemma}
\begin{proof}
For any time-step $t\in[T]$, define:
$$ M_t:= \sum_{\pi\in\Psi} Q_{t}(\pi)R(\pi) - r_t(a_t).$$
From \Cref{lem:conditional-reward}, we have that $\E[M_t|\Gamma_{t-1}]=0$. Also note that for any time-step $t$, we have $|M_t|\leq 1$. Now consider any time-step $t'\in\TsafeCheck$. From Azuma's inequality, with probability $1-\delta'/\lceil m(t')+\log_2(\tau_1)\rceil^3$, we get:
$$ \sum_{t=1}^{t'} M_t \leq \sqrt{2t'\ln\Bigg(\frac{(m(t')+\log_2(\tau_1))^3}{\delta'}\Bigg)}. $$
Any epoch $m$ has at most $\lceil m+\log_2(\tau_1) \rceil$ time-steps in $\TsafeCheck$. Therefore $\eventEvalRegret$ holds with probability at least:
$$ 1 - \sum_{t'\in\TsafeCheck} \frac{\delta'}{\lceil m(t')+\log_2(\tau_1)\rceil^3} \geq 1 - \sum_{m=1}^{\infty} \frac{\delta'}{(m+\log_2(\tau_1))^2} \geq 1 - \delta'. $$
\end{proof}

\subsection{Adapting \texttt{FALCON+}}


As we have stated before, both our algorithm and analysis builds on the work of \cite{simchi2020bypassing}. In this section, without assuming realizability, we show that the analysis of \texttt{FALCON+} continues to hold for the first few epochs of $\SafeFalcon$. The simple observation that allows us to make this extension is that \Cref{lem:reg-est-accuracy} can be proved without assuming realizability.

Lemma~\ref{lem:QmRegEst} is more or less a restatement of Lemma 5 in \cite{simchi2020bypassing}, and the proof stays the same. We only include the proof for completeness, as it states a key bound on the estimated regret of the randomized policy $Q_m$.

\begin{restatable}[]{lemma}{lemQmRegEst}
\label{lem:QmRegEst}
For any safe epoch $m$, we have:
$$ \sum_{\pi\in\Psi} Q_m(\pi)\Reg_{\hatf_m}(\pi) \leq \frac{K}{\gamma_m}. $$
\end{restatable}
\begin{proof}
This follows essentially from unpacking the definitions of regret \eqref{eq:instantaneous-regret} and the representation of the action selection kernel $p_m$ as in \eqref{eq:action_kernel} and \eqref{eq:connecting-p-and-Q}.
\begin{align*}
    & \sum_{\pi\in\Psi} Q_m(\pi)\Reg_{\hatf_m}(\pi) = \sum_{\pi\in\Psi} Q_m(\pi)\E_{x\sim D_{\Xscript}}\Big[ \hatf_m(x,\pi_{\hatf_m}(x)) - \hatf_m(x,\pi(x)) \Big]\\
    & = \E_{x\sim D_{\Xscript}}\Big[ \sum_{\pi\in\Psi} Q_m(\pi)\Big(\hatf_m(x,\pi_{\hatf_m}(x)) - \hatf_m(x,\pi(x))\Big) \Big]\\
    & = \E_{x\sim D_{\Xscript}}\Big[ \sum_{a\in\A} \sum_{\pi\in\Psi} \I(\pi(x)=a) Q_m(\pi)\Big(\hatf_m(x,\pi_{\hatf_m}(x)) - \hatf_m(x,a)\Big) \Big]\\
    & = \E_{x\sim D_{\Xscript}}\Big[ \sum_{a\in\A} p_m(a|x) \Big(\hatf_m(x,\pi_{\hatf_m}(x)) - \hatf_m(x,a)\Big) \Big]\\
    & = \E_{x\sim D_{\Xscript}}\Bigg[ \sum_{a\in\A} \frac{\Big(\hatf_m(x,\pi_{\hatf_m}(x)) - \hatf_m(x,a)\Big)}{K+\gamma_m\Big(\hatf_m(x,\pi_{\hatf_m}(x)) - \hatf_m(x,a)\Big)} \Bigg] \leq \frac{K}{\gamma_m}.
\end{align*}
\end{proof}


For any policy, \Cref{lem:reg-est-accuracy} bounds the prediction error of the implicit reward estimate for the first few epochs. This Lemma and its proof are similar to those of Lemma 12 in \cite{simchi2020bypassing}. Our definition of $\msafe$ and our choice of $\gamma_{m+1}$ allows us to prove this Lemma without assuming realizability.

\begin{restatable}{lemma}{lemImpPolEval}
\label{lem:reg-est-accuracy}
Suppose the event $\eventReg$ defined in \eqref{eq:w-event-Reg} holds. Then, for all policies $\pi$ and epochs $m\leq \min\{\msafe,\msafealg\}$, we have:
\begin{align*}
    |R_{\hatf_{m+1}}(\pi)-R(\pi)| \leq \frac{\sqrt{V(p_m,\pi)}\sqrt{K}}{2\gamma_{m+1}}
\end{align*}
\end{restatable}
\begin{proof}
For any policy $\pi$ and epoch $m\in [m^*]$, note that:
\begin{align*}
    &|R_{\hatf_{m+1}}(\pi)-R(\pi)|\\
    \leq & \E_{x\sim D_{\Xscript}}\Big[\Big|\hatf_{m+1}(x,\pi(x))-f^*(x,\pi(x))\Big|\Big] \\
    = & \E_{x\sim D_{\Xscript}}\Bigg[\sqrt{\frac{1}{p_m(\pi(x)|x)}p_m(\pi(x)|x)\Big(\hatf_{m+1}(x,\pi(x))-f^*(x,\pi(x))\Big)^2}\Bigg] \\
    \leq & \E_{x\sim D_{\Xscript}}\Bigg[\sqrt{\frac{1}{p_m(\pi(x)|x)}\E_{a\sim p_m(\cdot|x)}\Big[\Big(\hatf_{m+1}(x,a)-f^*(x,a)\Big)^2 \Big]}\Bigg] \\
    \leq & \sqrt{\E_{x\sim D_{\Xscript}}\Bigg[\frac{1}{p_m(\pi(x)|x)} \Bigg]} \sqrt{\E_{x\sim D_{\Xscript}} \E_{a\sim p_m(\cdot|x)}\Big[\Big(\hatf_{m+1}(x,a)-f^*(x,a)\Big)^2\Big]}\\
    \leq & \sqrt{V(p_m, \pi)} \sqrt{2\xi\Big(\tau_m-\tau_{m-1}, \frac{\delta'}{m^2} \Big)} = \frac{\sqrt{V(p_m,\pi)}\sqrt{K}}{2\gamma_{m+1}}.
\end{align*}
The first inequality follows from Jensen's inequality, the second inequality is straight forward, the third inequality follows from Cauchy-Schwarz inequality, and the last inequality follows from assuming that $\eventReg$ from \eqref{eq:w-event-Reg} holds. 
\end{proof}


The next Lemma implies that before misspecification becomes a problem we are able to bound regret in the same manner as \cite{simchi2020bypassing}. Note that for any epoch $m\leq \msafealg$, the action selection kernel used in epoch $m$ is given by \eqref{eq:action_kernel}. Further note that since the regression rates are valid (\Cref{ass:main-assumption}), from \eqref{eq:xi-first-condition} and \eqref{eq:gamma}, we have that $\gamma_m$ is increasing in $m$. Finally, since \Cref{lem:reg-est-accuracy} holds for all $m\leq \min\{\msafe,\msafealg\}$, following the proof of Lemma 13 in \cite{simchi2020bypassing}, we get: 

\begin{restatable}{lemma}{lemboundReg}
\label{lem:policy-reg-bound}
Suppose the event $\eventReg$ defined in \eqref{eq:w-event-Reg} holds. Let $C_0=5.15$. For all policies $\pi$ and epochs $m\leq \min\{\msafe,\msafealg\} + 1$, we have:
\begin{align*}
    \Reg(\pi) &\leq 2\Reg_{\hatf_{m}}(\pi) + \frac{C_0K}{\gamma_m}\\
    \Reg_{\hatf_{m}}(\pi) &\leq 2\Reg(\pi) + \frac{C_0K}{\gamma_m}
\end{align*}
\end{restatable}

That is, for any policy, \Cref{lem:policy-reg-bound} bounds the prediction error of the implicit regret estimate for the first few epochs. \Cref{lem:QmRegTrue} bounds the expected regret of the randomized policy $Q_m$ for the first few epochs. \Cref{lem:QmRegTrue} and its proof is more or less the same as the statement and the proof of Lemma 9 in \cite{simchi2020bypassing}.

\begin{restatable}{lemma}{lemQmRegTrue}
\label{lem:QmRegTrue}
Suppose the event $\eventReg$ defined in \eqref{eq:w-event-Reg} holds. Then for all epochs $m\leq \min\{\msafe,\msafealg\} + 1$, we have:
$$ \sum_{\pi \in \Psi} Q_m(\pi)\Reg(\pi) \leq \frac{(2+C_0)K}{\gamma_m}.$$
\end{restatable}
\begin{proof}
For any $m\leq \min\{\msafe,\msafealg\}$:
\begin{align*}
    \sum_{\pi \in \Psi} Q_m(\pi)\Reg(\pi)
    \leq \sum_{\pi \in \Psi} Q_m(\pi)\bigg( 2\Reg_{\hatf_{m}}(\pi) + \frac{C_0K}{\gamma_m} \bigg)
     \leq \frac{2K}{\gamma_m} + \frac{C_0K}{\gamma_m}.
\end{align*}
Where the first inequality follows from \Cref{lem:policy-reg-bound}, and the second inequality follows from Lemma \ref{lem:QmRegEst}.
\end{proof}

\subsection{Bounding $\msafealg$ and $l_m$}

\Cref{lem:bound-msafealg} shows that when the events defined in \Cref{sec:high-prob-events-1} hold, then $\msafealg$ is at least as large as $\msafe+1$. In particular, this means that $\msafe+1$ is deemed a safe epoch with high probability.

\begin{lemma}
\label{lem:bound-msafealg}
Suppose the events $\eventReg$, $\eventEval$, and $\eventEvalRegret$ hold. When $T\leq \tau_{\msafe+1}$, the status of the algorithm at the end of round $T$ is safe. When $T> \tau_{\msafe+1}$, we have that $\msafe + 1 \leq \msafealg$.
\end{lemma}
\begin{proof}
We first prove that under the assumptions of the theorem, $\msafe + 1 \leq \msafealg$ when $T> \tau_{\msafe+1}$. Suppose for contradiction that $T>\tau_{\msafe+1}$ and $\msafealg\leq \msafe$. Now consider the epoch $m'=\msafealg + 1$. By assumption and choice of $m'$, we have $\msafealg < m'$ and $m'\leq \min\{\msafe,\msafealg\}+1$. Therefore $m'$ is not a safe epoch, hence there exists a time-step $t'$ in epoch $m'$ such that $t'\in \TsafeCheck$ and we have that:
\begin{equation}
\label{eq:contradiction}
    \begin{aligned}
        \sum_{t=1}^{t'} r_t(a_t) < L_{t'}.
    \end{aligned}
\end{equation}
Since $\eventEvalRegret$ holds and $t'\in\TsafeCheck$, we have that:
\begin{equation}
\label{eq:cumregret-lower-bound}
    \begin{aligned}
        \sum_{t=1}^{t'} r_t(a_t) \geq \sum_{t=1}^{t'}\sum_{\pi\in\Psi} Q_t(\pi)R(\pi) - \sqrt{2t'\ln\Bigg(\frac{(m'+\log_2(\tau_1))^3}{\delta'}\Bigg)} .
    \end{aligned}
\end{equation}
For all $t\leq t'$, note that $m(t)\leq m'$. Therefore from \Cref{lem:QmRegTrue}, we have:
\begin{equation}
\label{eq:QmR-lower-bound}
    \begin{aligned}
        &\sum_{t=1}^{t'}\sum_{\pi\in\Psi} Q_t(\pi)R(\pi)\\
        & = t'R(\pi^*) - \sum_{t=1}^{t'}\sum_{\pi\in\Psi} Q_t(\pi)\Reg(\pi)\\
        & \geq t'\cdot l_{m'-1} - \tau_1 - \sum_{t=\tau_1+1}^{t'} \frac{(2+C_0)K}{\gamma_{m(t)}}\\
        & \geq t'\cdot l_{m'-1} - \tau_1 - 20.3\sqrt{K}\sum_{t=\tau_1+1}^{t'} \sqrt{\xi\Big(\tau_{m(t)-1}-\tau_{m(t)-2}, \frac{\delta'}{(m(t))^2}\Big)}.
    \end{aligned}
\end{equation}
Here, the first equality follows from the definition of $\Reg(\cdot)$. The first inequality follows from $\eventEval$ and \Cref{lem:QmRegTrue}. The result from \Cref{lem:QmRegTrue} can be used here since $\eventReg$ holds and since $m(t)\leq m' \leq \min\{\msafe,\msafealg\}+1$ for all time-steps $t\leq t'$. The last inequality follows from substituting the value for $C_0$ and $\gamma_{m(t)}$. Combining \eqref{eq:cumregret-lower-bound} and \eqref{eq:QmR-lower-bound} contradicts \eqref{eq:contradiction}. Therefore when $T> \tau_{\msafe+1}$, we have that $\msafe + 1 \leq \msafealg$.

The proof of the fact that the status of the algorithm at the end of round $T$ is safe when $T\leq \tau_{\msafe+1}$ is similar. Suppose for contradiction, $T\leq \tau_{\msafe+1}$ and the status at the end of round $T$ is not safe. We define $t'\in\TsafeCheck$ to be the first round where the status of the algorithm switches to ``not safe'' and we let $m'=m(t')$. Since $t'$ is the first such time-step we have that $m'=\msafealg+1$. Further, since $m'\leq m(T) \leq \msafe+1$, we again have $\msafealg\leq \msafe$ and $m'\leq \min\{\msafe,\msafealg\}+1$. Hence \eqref{eq:contradiction}, \eqref{eq:cumregret-lower-bound}, and \eqref{eq:QmR-lower-bound} still hold. Giving us the same contradiction, because combining \eqref{eq:cumregret-lower-bound} and \eqref{eq:QmR-lower-bound} contradicts \eqref{eq:contradiction}. This completes the proof of \Cref{lem:bound-msafealg}.
\end{proof}

For all $m\geq \msafe+1$, \Cref{lem:bound-l} lower bounds $l_m$ in terms of the optimal expected reward and the average misspecification error. Hence lower bounding the expected instantaneous reward of the algorithm when the status is ``not safe''.

\begin{lemma}
\label{lem:bound-l}
Suppose the events $\eventReg$ and $\eventEval$ hold. For any epoch $m\geq \msafe + 1$, we then have that:
$$ l_m \geq R(\pi^*) - 20.3\sqrt{KB} - \sqrt{2B}. $$
\end{lemma}
\begin{proof}
For compactness, let $S$ denote the set $S_{\msafe+1}$. From \Cref{lem:bound-msafealg}, we have that $\msafe+1$ is a safe epoch. Hence at any epoch $m\geq \msafe + 1$, from the update rule in $\ChoosePi$ we have:
\begin{equation}
\label{eq:lower-bound-lm}
    \begin{aligned}
        l_m &\geq l_{\msafe+1}\\
        &\geq \frac{1}{|S|}\sum_{(x,a,r)\in S} r \; - \sqrt{\frac{1}{2|S|}\ln\bigg(\frac{(\msafe+1)^2}{\delta'}\bigg)}\\
        &\geq \sum_{\pi\in\Psi} Q_{\msafe+1}(\pi)R(\pi) - \sqrt{\frac{2}{|S|}\ln\bigg(\frac{(\msafe+1)^2}{\delta'}\bigg)}\\
        &\geq R(\pi^*) - \frac{(2+C_0)K}{\gamma_{\msafe+1}} - \sqrt{\frac{2}{\tau_{\msafe+1}-\tau_{\msafe}}\ln\bigg(\frac{(\msafe+1)^2}{\delta'}\bigg)}.
    \end{aligned}
\end{equation}
Where the second last inequality follows from the fact that $\eventEval$ holds and the last inequality follows from \Cref{lem:QmRegTrue}. Now from the definition of $\msafe$, we have:
\begin{equation}
\label{eq:bound-K-by-gamma-star}
    \begin{aligned}
        &\xi\Big(\tau_{\msafe+1}-\tau_{\msafe}, \frac{\delta'}{(\msafe+1)^2} \Big) < B\\
        \implies & \frac{K}{\gamma_{\msafe+1}}\leq \sqrt{8K\xi\Big(\tau_{\msafe+1}-\tau_{\msafe}, \frac{\delta'}{(\msafe+1)^2} \Big)} < \sqrt{8KB}.
    \end{aligned}
\end{equation}
From \Cref{ass:main-assumption}, we have that $\xi$ is a valid rate, hence \eqref{eq:xi-second-condition} holds. Therefore, from \eqref{eq:xi-second-condition} and the definition of $\msafe$, we have:
\begin{equation}
\label{eq:lower-bound-eval-width-B}
    \begin{aligned}
        &\sqrt{\frac{2}{\tau_{\msafe+1}-\tau_{\msafe}}\ln\bigg(\frac{(\msafe+1)^2}{\delta'}\bigg)} \leq \sqrt{2\xi\Big(\tau_{\msafe+1}-\tau_{\msafe}, \frac{\delta'}{(\msafe+1)^2} \Big)} < \sqrt{2B}.
    \end{aligned}
\end{equation}
The result follows from combining \eqref{eq:lower-bound-lm}, \eqref{eq:bound-K-by-gamma-star}, and \eqref{eq:lower-bound-eval-width-B}.
\end{proof}

\subsection{Additional high probability events}

In this section, we show that events $\eventEvalcheck$ and $\eventAzuma$ hold with high-probability. The event $\eventEvalcheck$ provide upper and lower bound the difference between the expected regret and average regret at epochs that begin with a ``safe'' status.
\begin{equation}
    \label{eq:w-event-Eval-check}
    \begin{aligned}
      \eventEvalcheck := \Bigg\{ \forall t' \in [\TsafeCheck], \; &\sqrt{\frac{2}{t'-\tau_{m(t')-1}}\ln\bigg(\frac{\lceil m(t')+\log_2(\tau_1) \rceil^3}{\delta'}\bigg)} & \\ 
      & \geq \sum_{\pi\in\Psi} Q_{t'}(\pi)\Reg(\pi) - \frac{1}{t'-\tau_{m(t')-1}}\sum_{t=\tau_{m(t')-1}+1}^{t'} (r_t(\pi^*(x_t))- r_t(a_t)) &\\
      & \geq - \sqrt{\frac{2}{t'-\tau_{m(t')-1}}\ln\bigg(\frac{\lceil m(t')+\log_2(\tau_1)\rceil^3}{\delta'}\bigg)} & \Bigg\}. 
    \end{aligned}
\end{equation}

\begin{lemma}
The event $\eventEvalcheck$ holds with probability at least $1-2\delta'$.
\end{lemma}
\begin{proof}
Consider any time-step $t'\in\TsafeCheck$. Similar to the argument in \Cref{lem:high-prob-eventReg}, for the purpose of estimating the expected reward of the randomized policy $Q_{t'}$, we can treat the data generated in the first $t'-\tau_{m(t')-1}$ time-steps of epoch $m(t')$ as iid samples from the distribution $D(p_{m(t')})$.

Since $t'\in \TsafeCheck$, for all time-steps $t\leq t'$ from epoch $m(t')$, status is ``safe'' and actions are chosen according to the action selection kernel $p_{m(t')}$. Hence from Hoeffding's inequality, with probability at least $1-2\delta'/\lceil m(t')+\log_2(\tau_1)\rceil^3$, we get:
\begin{equation*}
    \begin{aligned}
      &\sqrt{\frac{2}{t'-\tau_{m(t')-1}}\ln\bigg(\frac{\lceil m(t')+\log_2(\tau_1)\rceil^3}{\delta'}\bigg)}  \\ 
      & \geq \sum_{\pi\in\Psi} Q_{t'}(\pi)\Reg(\pi) - \frac{1}{t'-\tau_{m(t')-1}}\sum_{t=\tau_{m(t')-1}+1}^{t'} (r_t(\pi^*(x_t))- r_t(a_t)) \\
      & \geq - \sqrt{\frac{2}{t'-\tau_{m(t')-1}}\ln\bigg(\frac{\lceil m(t')+\log_2(\tau_1)\rceil^3}{\delta'}\bigg)}. 
    \end{aligned}
\end{equation*}
Any epoch $m$ has at most $\lceil m+\log_2(\tau_1)\rceil$ time-steps in $\TsafeCheck$. Therefore $\eventEvalcheck$ holds with probability at least:
$$ 1 - \sum_{t'\in\TsafeCheck} \frac{2\delta'}{\lceil m(t')+\log_2(\tau_1)\rceil^3} \geq 1 - \sum_{m=1}^{\infty} \frac{2\delta'}{(m+\log_2(\tau_1))^2} \geq 1 - 2\delta'. $$
\end{proof}

The event $\eventAzuma$ provides lower and upper bounds on the cumulative reward of the algorithm and optimal policy for various ranges of time-steps.
\begin{equation}
    \label{eq:w-event-Azuma}
    \begin{aligned}
      \eventAzuma := \Bigg\{ \forall t' \in [\TsafeCheck]\cup\{T\}, \;  &\sum_{t=1}^{t'} (r_t(\pi^*(x_t)) - R(\pi^*)) \leq \sqrt{2t'\ln\bigg(\frac{\lceil m(t')+\log_2(\tau_1) \rceil^3}{\delta'}\bigg)} \\ 
      & \sum_{t=1}^{t'} (\sum_{\pi\in\Psi}Q_t(\pi)R(\pi) - r_t(a_t)) \leq \sqrt{2t'\ln\bigg(\frac{\lceil m(t')+\log_2(\tau_1) \rceil^3}{\delta'}\bigg)} \\
      &\sum_{t=t'+1}^{T} (r_t(\pi^*(x_t)) - R(\pi^*)) \leq \sqrt{2(T-t')\ln\bigg(\frac{\lceil m(t')+\log_2(\tau_1) \rceil^3}{\delta'}\bigg)} \\ 
      & \sum_{t=t'+1}^{T} (\sum_{\pi\in\Psi}Q_t(\pi)R(\pi) - r_t(a_t)) \leq \sqrt{2(T-t')\ln\bigg(\frac{\lceil m(t')+\log_2(\tau_1) \rceil^3}{\delta'}\bigg)} \Bigg\}.
    \end{aligned}
\end{equation}

\begin{lemma}
The event $\eventAzuma$ holds with probability at least $1-4\delta'$.
\end{lemma}
\begin{proof}
For each round $t$, we define:
\begin{align*}
    &M_t := r_t(\pi^*(x_t)) - R(\pi^*),\\
    &M'_t := \sum_{\pi\in\Psi}Q_t(\pi)R(\pi) - r_t(a_t).
\end{align*}
It is straightforward to see that $\E[M_t|\Gamma_{t-1}]=0$. Further from \Cref{lem:conditional-reward}, we have that $\E[M_t'|\Gamma_{t-1}]=0$. Now consider any time-step $t'\in\TsafeCheck \cup \{T\}$.  Since $|M_t|,|M_t'|\leq 1$ for all $t$, from Azuma's inequality, with probability at least $1-4\delta'/\lceil m(t')+\log_2(\tau_1) \rceil^3$, we have: \footnote{For $t'=T$, the last two inequalities in \eqref{eq:azuma-applied} are trivial.}
\begin{equation}
\label{eq:azuma-applied}
    \begin{aligned}
        &\sum_{t=1}^{t'}M_t \leq \sqrt{2t'\ln\bigg(\frac{\lceil m(t')+\log_2(\tau_1) \rceil^3}{\delta'}\bigg)},\\
        &\sum_{t=1}^{t'}M'_t \leq \sqrt{2t'\ln\bigg(\frac{\lceil m(t')+\log_2(\tau_1) \rceil^3}{\delta'}\bigg)},\\
        &\sum_{t=t'+1}^{T}M_t \leq \sqrt{2(T-t')\ln\bigg(\frac{\lceil m(t')+\log_2(\tau_1) \rceil^3}{\delta'}\bigg)},\\
        &\sum_{t=t'+1}^{T}M'_t \leq \sqrt{2(T-t')\ln\bigg(\frac{\lceil m(t')+\log_2(\tau_1) \rceil^3}{\delta'}\bigg)}.
    \end{aligned}
\end{equation}
Any epoch $m$ has at most $\lceil m+\log_2(\tau_1)\rceil$ time-steps in $\TsafeCheck \cup \{T\}$. Therefore $\eventAzuma$ holds with probability at least:
$$ 1 - \sum_{t'\in\TsafeCheck} \frac{4\delta'}{\lceil m(t')+\log_2(\tau_1) \rceil^3} \geq 1 - \sum_{m=1}^{\infty} \frac{4\delta'}{( m(t')+\log_2(\tau_1) )^2} \geq 1 - 4\delta'. $$
\end{proof}

\subsection{Proof of Theorem \ref{thm:main-theorem}}
\label{app:proof-of-main-theorem}

Note that $\eventReg$, $\eventEval$, $\eventEvalRegret$, $\eventEvalcheck$, and $\eventAzuma$ all hold with probability at least $1-\delta$. We now split our analysis into two cases and bound the cumulative regret for each case, while assuming all these high-probability events hold.

\textbf{Case 1} ($T\leq \tau_{\msafe+1}$):\\ 
From $\eventAzuma$, we have that:
\begin{equation}
\label{eq:bound-cummulative-regret-with-expected}
\begin{aligned}
    \sum_{t=1}^T \Big( r_t(\pi^*(x_t)) - r_t(a_t) \Big)
    \leq \sum_{t=1}^T\sum_{\pi\in\Psi}Q_t(\pi)\Reg(\pi) + \sqrt{8T\ln\bigg(\frac{\lceil m(T)+\log_2(\tau_1) \rceil^3}{\delta'}\bigg)}.
\end{aligned}
\end{equation}
Since $\eventReg$, $\eventEval$, $\eventEvalRegret$ hold, from \Cref{lem:bound-msafealg} we have that the status at the end of round $T$ is safe. Since $\eventReg$ holds, from \Cref{lem:QmRegTrue}, we have:
\begin{equation}
\label{eq:bound-case-1}
\begin{aligned}
    &\sum_{t=1}^T\sum_{\pi\in\Psi}Q_t(\pi)\Reg(\pi) \\
    &\leq \tau_1 + \sum_{t=\tau_1+1}^{T} \frac{(2+C_0)K}{\gamma_{m(t)}}\\
    &\leq \tau_1 + 20.3\sqrt{K}\sum_{t=\tau_1+1}^{T} \sqrt{\xi\Big(\tau_{m(t)-1}-\tau_{m(t)-2}, \frac{\delta'}{(m(t))^2}\Big)}.
\end{aligned}
\end{equation}
Combining \cref{eq:bound-cummulative-regret-with-expected} and \cref{eq:bound-case-1} completes the analysis for the first case.

\textbf{Case 2} ($T> \tau_{\msafe+1}$):\\ 
Let $t'$ be the last time-step where the conditions in $\IsSafe$ were checked and verified to be true. That is, \footnote{In the definition of $t'$ (see \eqref{eq:t-prime-def}), it may seem redundant to consider the set $\TsafeCheck \cap [T]$ when $\TsafeCheck\subseteq [T]$. We do this because, later in the proof, we use a thought experiment where we make the bandit run for more that $T$ time-steps.  }
\begin{equation}
\label{eq:t-prime-def}
    t'=\max\{t\in \TsafeCheck \cap [T] |\text{ algorithm status is safe at the end of round $t$} \}.
\end{equation}
Since $\eventReg$, $\eventEval$, $\eventEvalRegret$ hold, from \Cref{lem:bound-msafealg}, we have that epoch $\msafe+1$ is safe. Therefore, the last round of this epoch is safe and hence $t' \geq \tau_{\msafe+1}$. We now prove a bound on the cumulative regret up to time $t'$ (see \eqref{eq:bound-cum-regret-t-prime}). We start by deriving a bound for the cumulative regret up to time $\tau_{\msafe+1}$ and then derive a bound for the cumulative regret up to time $t'$ when $t'>\tau_{\msafe+1}$.

Since $\tau_{\msafe+1}$ lies in $\TsafeCheck$, the event $\eventAzuma$ bounds the cumulative regret upto time $\tau_{\msafe+1}$ in terms of the expected cumulative regret. Hence following the arguments in case 1, we get:
\begin{equation}
\label{eq:bound-cum-regret-upto-tau-msafe-plus-one}
\begin{aligned}
    &\sum_{t=1}^{\tau_{\msafe+1}} \Big( r_t(\pi^*(x_t)) - r_t(a_t) \Big) \leq \tau_1 + 20.3\sqrt{K}\sum_{t=\tau_1+1}^{\tau_{\msafe+1}} \sqrt{\xi\Big(\tau_{m(t)-1}-\tau_{m(t)-2}, \frac{\delta'}{(m(t))^2}\Big)}.
\end{aligned}
\end{equation}
Hence, we now only need to bound the cumulative regret up to time $t'$ when $t'>\tau_{\msafe+1}$. For compactness, let $m'$ denote the epoch $m(t')$. Since the status of the algorithm is safe at the end of round $t'$ and $t'\in\TsafeCheck$, we have that:
\begin{equation}
\label{eq:last-safe-time-step}
    \begin{aligned}
        &\sum_{t=1}^{t'} r_t(a_t) \geq t'\cdot l_{m'-1} - \tau_1 - \sqrt{2t'\ln\Big(\frac{(m'+\log_2(\tau_1))^3}{\delta'}\Big)} - 20.3\sqrt{K}\sum_{i=\tau_2}^{t'} \sqrt{\xi\Big(\tau_{m(i)-1}-\tau_{m(i)-2}, \frac{\delta'}{(m(i))^2}\Big)}.
    \end{aligned}
\end{equation}
Now note that when $t'>\tau_{\msafe+1}$, we have:
\begin{equation}
\label{eq:bound-cum-regret-t-prime-case-2}
    \begin{aligned}
        &\sum_{t=1}^{t'} (r_t(\pi^*(x_t)) - r_t(a_t))\\
        & \leq \sum_{t=1}^{t'} (R(\pi^*) - r_t(a_t)) + \sqrt{2t'\ln\bigg(\frac{\lceil m(t')+\log_2(\tau_1) \rceil^3}{\delta'}\bigg)} \\
        & \leq t'\cdot (20.3\sqrt{KB}+\sqrt{2B}) + \tau_1 + \sqrt{8t'\ln\bigg(\frac{\lceil m(t')+\log_2(\tau_1) \rceil^3}{\delta'}\bigg)}  \\
        & \;\;\;\; + 20.3\sqrt{K}\sum_{i=\tau_2}^{t'} \sqrt{\xi\Big(\tau_{m(i)-1}-\tau_{m(i)-2}, \frac{\delta'}{(m(i))^2}\Big)}.
    \end{aligned}
\end{equation}
Where the first inequality follows from the fact that $\eventAzuma$ holds. When $t'>\tau_{\msafe+1}$, note that $m'> \msafe+1$, hence \Cref{lem:bound-l} bounds $l_{m'-1}$. The last inequality follows from \eqref{eq:last-safe-time-step}, and \Cref{lem:bound-l}.  

To recap, we already argued that $t'\geq \tau_{\msafe+1}$. In \eqref{eq:bound-cum-regret-upto-tau-msafe-plus-one}, we also bounded cumulative regret up to time $\tau_{\msafe+1}$. Finally, \eqref{eq:bound-cum-regret-t-prime-case-2} bounds cumulative regret up to time $t'$ when $t'>\tau_{\msafe+1}$. Combining everything together, we get an unconditional bound on the cumulative regret up to time $t'$:
\begin{equation}
\label{eq:bound-cum-regret-t-prime}
    \begin{aligned}
        &\sum_{t=1}^{t'} (r_t(\pi^*(x_t)) - r_t(a_t))\\
        & \leq t'\cdot (20.3\sqrt{KB}+\sqrt{2B}) + \tau_1 + \sqrt{8t'\ln\bigg(\frac{\lceil m(t')+\log_2(\tau_1) \rceil^3}{\delta'}\bigg)}  \\
        & \;\;\;\; + 20.3\sqrt{K}\sum_{i=\tau_2}^{t'} \sqrt{\xi\Big(\tau_{m(i)-1}-\tau_{m(i)-2}, \frac{\delta'}{(m(i))^2}\Big)}.
    \end{aligned}
\end{equation}
Note that \eqref{eq:bound-cum-regret-t-prime} bounds the cumulative regret up to the last time-step where the conditions in $\IsSafe$ were verified to be true.
Also if $t'=T$, then \eqref{eq:bound-cum-regret-t-prime} gives us the required bound on the cumulative regret up to time $T$. Hence, moving forward, we only need to bound cumulative regret when $T > t'$.

Recall that we defined $t'$ to be the last time-step where the conditions in $\IsSafe$ were verified to be true. Let $t''$ be the next time-step where the $\IsSafe$ conditions would be checked. We now bound the cumulative regret up to time $t''$ in terms of the cumulative regret up to time $t'$.

Note that if $T\geq t''$, then $t''\in \TsafeCheck$. On the other hand if $T<t''$, it is easy to see that the cumulative regret up to time $T$ would be roughly smaller than the cumulative regret up to time $t''$ had the bandit run up to round $t''$. Since we are going to bound the cumulative regret up to time $t''$ in terms of the cumulative regret up to time $t'$ and $T>t'$, for the purposes of bounding cumulative regret up to time $T$, we can assume that the bandit runs up to time $t''$ and $t''\in\TsafeCheck$.   

Note that if $t''=t'+1$, we have:
\begin{equation}
\label{eq:bound-t-double-trivial}
    \sum_{t=1}^{t''} (r_t(\pi^*(x_t)) - r_t(a_t)) \leq 2 + \sum_{t=1}^{t'} (r_t(\pi^*(x_t)) - r_t(a_t)) 
\end{equation}

We now want to bound the cumulative regret up to time $t''$ in terms of the cumulative regret up to time $t'$ when $t''>t'+1$. Since both $t'$ and $t''$ are consecutive rounds in $\TsafeCheck$, when $t''> t'+1$, we have that both rounds lie in the same epoch. That is, $m(t'')=m(t')=m'$. When both rounds lie in the same epoch, since the status of the algorithm is safe at the end of round $t'$, we have that the action selection kernel $p_{m'}$ is used to pick actions at every time-step $t\in [\tau_{m'-1}+1,t'']$. Therefore, when $t''>t'+1$, we have:
\begin{equation}
\label{eq:bound-t-double}
    \begin{aligned}
        &\sum_{t=\tau_{m'-1}+1}^{t''} (r_t(\pi^*(x_t))-r_t(a_t))\\
        &\leq (t''-\tau_{m'-1})\sum_{\pi\in\Psi} Q_{m'}(\pi)\Reg(\pi) + \sqrt{2(t''-\tau_{m'-1})\ln\bigg(\frac{\lceil m'+\log_2(\tau_1)\rceil^3}{\delta'}\bigg)}\\
        &\leq 2(t'-\tau_{m'-1})\sum_{\pi\in\Psi} Q_{m'}(\pi)\Reg(\pi) + \sqrt{4(t'-\tau_{m'-1})\ln\bigg(\frac{\lceil m'+\log_2(\tau_1)\rceil^3}{\delta'}\bigg)}\\
        &\leq 2\sum_{t=\tau_{m'-1}+1}^{t'} (r_t(\pi^*(x_t))-r_t(a_t)) + 5 \sqrt{(t'-\tau_{m'-1})\ln\bigg(\frac{\lceil m'+\log_2(\tau_1)\rceil^3}{\delta'}\bigg)}.
    \end{aligned}
\end{equation}
The following arguments assume $t''>t'+1$. The first inequality follows from the fact that $t''\in\TsafeCheck$, the fact that $\eventEvalcheck$ holds, and the fact that the algorithm uses the action selection kernel $p_{m'}$ to pick actions at every time-step $t\in [\tau_{m'-1}+1,t'']$. The second inequality follows from the fact that $t',t''$ are consecutive rounds in $\TsafeCheck$, and are in the same epoch. The last inequality follows from the fact that $t'\in\TsafeCheck$, $\eventEvalcheck$ holds, and the fact that the algorithm uses the action selection kernel $p_{m'}$ to pick actions at every time-step $t\in [\tau_{m'-1}+1,t']$.

Therefore, when $t''>t'+1$, from \eqref{eq:bound-t-double} we have:
\begin{equation}
\label{eq:bound-t-double-nontrivial}
    \begin{aligned}
        &\sum_{t=1}^{t''} (r_t(\pi^*(x_t))-r_t(a_t)) \leq 3\sum_{t=1}^{t'} (r_t(\pi^*(x_t))-r_t(a_t)) + 5 \sqrt{T\ln\bigg(\frac{(m'+\log_2(\tau_1))^3}{\delta'}\bigg)}
    \end{aligned}
\end{equation}

To recap, we know that $t''>t'$. When $t''=t'+1$, \eqref{eq:bound-t-double-trivial} bounds the cumulative regret up to time $t''$ in terms of the cumulative regret up to time $t'$. When $t''>t'+1$, \eqref{eq:bound-t-double-nontrivial} bounds the cumulative regret up to time $t''$ in terms of the cumulative regret up to time $t'$. Combining everything together, we get the following unconditional bound on the cumulative regret up to time $t''$ in terms of the cumulative regret up to time $t'$:
\begin{equation}
\label{eq:bound-cumlative-regret-t-double}
    \begin{aligned}
        &\sum_{t=1}^{t''} (r_t(\pi^*(x_t))-r_t(a_t)) \leq 2 + 3\sum_{t=1}^{t'} (r_t(\pi^*(x_t))-r_t(a_t)) + 5 \sqrt{T\ln\bigg(\frac{(m'+\log_2(\tau_1))^3}{\delta'}\bigg)}
    \end{aligned}
\end{equation}

\textbf{Case 2.1} ($T> \tau_{\msafe+1}$ and $T<t''$):\\
Note that:
\begin{equation}
\label{eq:case-2-1}
    \begin{aligned}
        &\sum_{t=1}^T \Big( r_t(\pi^*(x_t)) - r_t(a_t) \Big)\\
        &\leq \sum_{t=1}^T\sum_{\pi\in\Psi}Q_t(\pi)\Reg(\pi) + \sqrt{8T\ln\bigg(\frac{\lceil m(T)+\log_2(\tau_1) \rceil^3}{\delta'}\bigg)}\\
        &\leq \sum_{t=1}^{t''}\sum_{\pi\in\Psi}Q_t(\pi)\Reg(\pi) + \sqrt{8T\ln\bigg(\frac{\lceil m(T)+\log_2(\tau_1) \rceil^3}{\delta'}\bigg)}\\
        &\leq \sum_{t=1}^{t''} \Big( r_t(\pi^*(x_t)) - r_t(a_t) \Big) + 7\sqrt{T\ln\bigg(\frac{\lceil 1+m(T)+\log_2(\tau_1) \rceil^3}{\delta'}\bigg)}.
    \end{aligned}
\end{equation}
Where the first inequality follows from the fact that $\eventAzuma$ holds. The second inequality follows from the fact that $t''>T$ and the fact that $\Reg(\pi)\geq 0$. The last inequality follows from $\eventAzuma$ and the fact that $t''\leq 2t' \leq 2T$.

Combining \eqref{eq:bound-cum-regret-t-prime}, \eqref{eq:bound-cumlative-regret-t-double}, and \eqref{eq:case-2-1} gives us the required bound on the cumulative regret for this case.

\textbf{Case 2.2} ($T> \tau_{\msafe+1}$ and $T \geq t''$):\\
From the definition of $t'$ and $t''$, the status of the algorithm switches to ``not safe'' at the end of round $t''$. Thereafter, all actions will be selected according to the action selection kernel $p_{\msafealg}$. Now note that:
\begin{equation}
\label{eq:case-2-2}
    \begin{aligned}
        &\sum_{t=t''+1}^T \Big( r_t(\pi^*(x_t)) - r_t(a_t) \Big)\\
        &\leq \sum_{t=t''+1}^T\sum_{\pi\in\Psi}Q_{\msafealg}(\pi)\Reg(\pi) + \sqrt{8(T-t'')\ln\bigg(\frac{\lceil m(t'')+\log_2(\tau_1) \rceil^3}{\delta'}\bigg)}\\
        &\leq (R(\pi^*)-l_{\msafealg})(T-t'') + \sqrt{8(T-t'')\ln\bigg(\frac{\lceil m(t'')+\log_2(\tau_1) \rceil^3}{\delta'}\bigg)}\\
        &\leq (20.3\sqrt{KB}+\sqrt{2B})(T-t'') + \sqrt{8(T-t'')\ln\bigg(\frac{\lceil m(t'')+\log_2(\tau_1) \rceil^3}{\delta'}\bigg)}.
    \end{aligned}
\end{equation}
Where the first inequality follows from $\eventAzuma$ and the fact that the kernel $p_{\msafealg}$ is used for all rounds $t\in [t''+1,T]$. The second inequality follows from $\eventEval$, which gives us that the expected reward of the randomized policy $Q_{\msafealg}$ is lower bounded by $l_{\msafealg}$. From \Cref{lem:bound-msafealg} we have that $\msafealg\geq \msafe+1$. Therefore the final inequality follows from \Cref{lem:bound-l}, which gives us a lower bound on $l_{\msafealg}$ since $\msafealg\geq \msafe+1$.

Combining \eqref{eq:bound-cum-regret-t-prime}, \eqref{eq:bound-cumlative-regret-t-double}, and \eqref{eq:case-2-2} gives us the required bound on the cumulative regret for this case.

\section{Proof for lower bound}
\label{app:proof-lower}

In this section, we prove \Cref{thm:lower-bound}, which we restate bellow for convenience.

\thmLbound*
\begin{proof}
Consider any $K\geq 2$ and $B\in[0,1/(2K)]$. We start by constructing a K arm stochastic contextual bandit instance. Let $\A=[K]$ be the set of arms, $\Xscript = (0,K)$ be the set of contexts, and $D$ denote the joint distribution of rewards and contexts. Such that $D_{\Xscript}$, the marginal distribution of $D$ on the set of contexts, is uniformly distributed over over the context set $\Xscript$. That is, $D_{\Xscript}\equiv \Unif(0,K)$. For all $x\in\Xscript$ and $a\in\A$, we let the conditional expected reward be given by,
\begin{equation}
    \label{eq:conditional-reward-lb}
    f^*(x,a):=
   \begin{cases}
    \alpha &\text{for } x\in (a-1,a],\\
    0 &\text{otherwise}.
   \end{cases}
\end{equation}
Where $\alpha$ is given by, 
\begin{equation}
    \label{eq:alpha}
    \alpha := \sqrt{\frac{K^2}{K-1}B}.
\end{equation}
Since $K\geq 2$ and $B\in [0,1/(2K)]$, we have that $\alpha\in [0,1]$. Hence the conditional expected reward for every context and action lies in $[0,1]^K$. We also let the rewards be noiseless. That is, 
\begin{equation}
    \Pr_{(x,r)\sim D}\big(r(a) = f^*(x,a)\big) = 1.
\end{equation}
This completes our description of the stochastic contextual bandit setup that we consider. We now let the model class $\F$ be given by,
\begin{equation}
    \label{eq:model-class-lb}
    \F = \{ f:\Xscript\times\A \rightarrow [0,1] \; | \; \forall a\in\A, \; \exists f_a\in [0,1] \text{ such that}, \; \forall x\in\Xscript, \; \text{ we have } f(x,a) = f_a  \}.
\end{equation}
That is, $\F$ is the class of all models that do not depend on contexts. Therefore, the set $\kernelSet$ contains all probability kernels that do not depend on contexts. 
\begin{equation}
    \label{eq:kernel-set-lb-case}
    \begin{aligned}
    \kernelSet=&\bigg\{p \text{ is a probability kernel} |\; \exists f_p\in\F, \; g_p:\A\times \R^{\A} \rightarrow [0,1], \; \forall \; (x,a)\in\Xscript \times \A, \; p(a|x) = g_p(a| f_p(x)) \bigg\}\\
    = & \bigg\{p \text{ is a probability kernel} |\; \exists \; g_p:\A \rightarrow [0,1], \; \forall \; (x,a)\in\Xscript \times \A, \; p(a|x) = g_p(a) \bigg\}.
    \end{aligned}
\end{equation}
That is, the set $\kernelSet$ simply reduces to the set of distributions over $\A$ Which also implies that $\kernelSet$ is convex. For notational convenience, we let $\G$ denote the set of all distributions over the action set $\A$. That is,
\begin{equation}
    \label{eq:action-distribution}
    \begin{aligned}
    \G = & \bigg\{g:\A \rightarrow [0,1] |\; \sum_{a\in\A} g(a) =1 \bigg\}.
    \end{aligned}
\end{equation}
Now note that for any arm $a\in\A$, from \eqref{eq:conditional-reward-lb}, we have that:
\begin{equation}
    \label{eq:avg-reward}
    \begin{aligned}
    \E_{x\sim D_{\Xscript}}[f^*(x,a)]=\frac{\alpha}{K}  \in[0,1].
    \end{aligned}
\end{equation}
Now note that, the average misspecification is given by:
\begin{equation}
    \begin{aligned}
    &\sqrt{\max_{p \in \kernelSet } \min_{f\in\F} \E_{x\sim D_{\Xscript}}\E_{a\sim p(\cdot|x)} [(f(x,a)-f^*(x,a))^2]}\\
    = & \sqrt{\max_{g \in \G } \sum_{a\in\A} g(a) \min_{f_a\in [0,1]} \E_{x\sim D_{\Xscript}}[(f_a-f^*(x,a))^2]}\\
    = & \sqrt{\max_{g \in \G } \sum_{a\in\A} g(a) \E_{x\sim D_{\Xscript}}[(f^*(x,a)-\alpha/K)^2]}\\
    = & \sqrt{\max_{g \in \G } \sum_{a\in\A} g(a) \frac{K-1}{K^2}\alpha} = \sqrt{B}.
    \end{aligned}
\end{equation}
Here the first equality follows from \eqref{eq:kernel-set-lb-case} and \eqref{eq:action-distribution}. The second equality follows from \eqref{eq:avg-reward} and the fact that the mean minimizes the mean squared error. The third equality follows from substituting the value for $f^*(x,a)$ from \eqref{eq:conditional-reward-lb}. Finally, the last equality follows from \eqref{eq:alpha} and the fact that $g\in\G$.

It is easy to see that the optimal policy is given by $\pi^*$ (defined in \eqref{eq:opt-policy-lb}). Further, the expected reward of $\pi^*$ is $\alpha$. That is, $R(\pi^*)=\alpha$. 
\begin{equation}
    \label{eq:opt-policy-lb}
    \text{For all $a\in\A$ and $x\in(a-1,a]$, } \pi^*(x) := a. 
\end{equation}

Now consider any arm $a\in\A$. With some abuse of notation, let $a$ also denote the policy that selects arm~$a$ for all contexts. Note that the expected reward for this policy is $\alpha/K$, that is $R(a)=\alpha/K$. For any probability kernel $p$ in $\kernelSet$, since it does not depend on contexts, the randomized policy $Q_p$ is only supported by policies in $\A$. Therefore, we have that the expected regret of any randomized policy $Q_p$ that is induced by some probability kernel in $\kernelSet$ is given by:
\begin{equation}
    \begin{aligned}
    &\sum_{\pi\in\Psi} Q_p(\pi)\Reg(\pi) =  \sum_{a\in\A} Q_p(a) \Reg(a) = \sum_{a\in\A} Q_p(a) \big( R(\pi^*) - R(a) \big) = \sum_{a\in\A} Q_p(a) \Big( 1-\frac{1}{K} \Big) \alpha\\
    & = \Big( 1-\frac{1}{K} \Big) \alpha = \sqrt{(K-1)B} \geq \sqrt{KB/2}.
    \end{aligned}
\end{equation}
This completes the proof of \Cref{thm:lower-bound}.
\end{proof}

\section{Detailed introduction example}
\label{sec:introduction_example}

To generate both \Cref{fig:falcon_failure} and \Cref{fig:safe-falcon}, we implement \texttt{FALCON+} and $\SafeFalcon$ respectively. Both implementations require knowledge of the estimation rate function ($\xi$). In this section, we detail our choice of estimation rates and give some more intuition for the oscillatory regret behavior we see in that \Cref{fig:falcon_failure}. 


\paragraph{FALCON+ Setup} As explained in the introduction, we use the \texttt{FALCON+} algorithm in \cite{simchi2020bypassing}. This algorithm requires knowledge of a function $\xi_{\F, \delta}(n)$ representing the estimation rate. This is defined in the following assumption.

\textit{Assumption 2 in \cite{simchi2020bypassing}} 
Given $n$ data samples $\{(x_1, a_1, r_1(a_1)), \cdots, (x_n, a_n, r_n(a_n)) \}$ generated \textit{iid} from an arbitrary distribution $\mathcal{D}_{\text{data}}$, the offline regression oracle return a function $\hat{f}$. For all $\delta > 0$, with probability at least $1 - \delta$, we have
\begin{equation}
    \label{eq:falcon_xi}
    \E[(\hat{f}(x, a) - f^*(x,a))^2] 
    \leq \xi_{\F, \delta}(n)
\end{equation}

We set the class of functions $\F$ to be the set of linear functions (i.e., linear regressions of outcomes on contexts and an intercept). For this model, it is straightforward to show analytically that the random variable on the left-hand side of \eqref{eq:falcon_xi} is distributed as a random variable $\frac{1}{n} \chi_2^2$, where $\chi_2^2$ is a random variable distributed as chi-squared with two degrees of freedom. Therefore, we set $\xi_{\F, \delta}(n)$ to be the $1-\delta$-quantile of the distribution of $\frac{1}{n} \chi_2^2$. Note that this quantity is decreasing with $n$. 

\paragraph{Explaining results} In Figure \ref{fig:falcon_failure}, we saw that average per-epoch regret oscillates between very low and very high levels. Let's explain this phenomenon.

First, note that the optimal treatment assignment policy for this setting is 
\begin{equation}
    \pi^*(x) = 
    \begin{cases}
        1 \qquad \text{if }x \leq .5, \\
        2 \qquad \text{otherwise}.
    \end{cases}
\end{equation}

Moreover, if arms were assigned uniformly at random, the the best linear approximation to $f^*(x, a)$ would be given by
\begin{align}
    \hat{f}^*(x, a) := 
    \begin{cases}
        -.25 + 1.5x \qquad \text{if } a = 1, \\
        .5 \qquad \text{if } a = 2. \\
    \end{cases}
\end{align}

Finally, note that the data-generating process was selected so that the policy $\pi_{\hat{f}^*}$ induced by the best linear approximation coincides with $\pi^*(x)$. 

With this in mind, we are ready to understand the oscillatory behavior in Figure \ref{fig:oscillation}. During the first seven or so epochs, arms are assigned roughly at random, and the linear outcome models fit on this data are not too different from $\hat{f}^*(x, a)$, but the exploitation parameter $\gamma_m \propto \xi_{\F, \delta}(n)^{-1}$ is small, so the induced action selection kernel $p_{m}(a|x)$ does not concentrate very much.

However, since epochs have increasing size, for later epochs $\gamma_m \propto \xi_{\F, \delta}(n)^{-1}$ can be large, and $p_m(a|x)$ concentrates and becomes approximately $\pi_{\hat{f}^*}$. This is what happens in Epoch 8 in Figure \ref{fig:falcon_failure}, for example. However, this skews the data distribution, so that in Epoch 9 when we fit an outcome model using data from Epoch 8 we get something that is very different from $\hat{f}^*(x, a)$. In turn, this accidentally increases the amount of exploration happening in this epoch. So in Epoch 10 we are able to return to an outcome model that is similar to  $\hat{f}^*(x, a)$. But by then the exploitation parameter $\gamma_m$ is even larger, so $p_m(a|x)$ concentrates again, causing the cycle to repeat. 

\begin{figure}[H]
    \centering
    \includegraphics[width=1\textwidth]{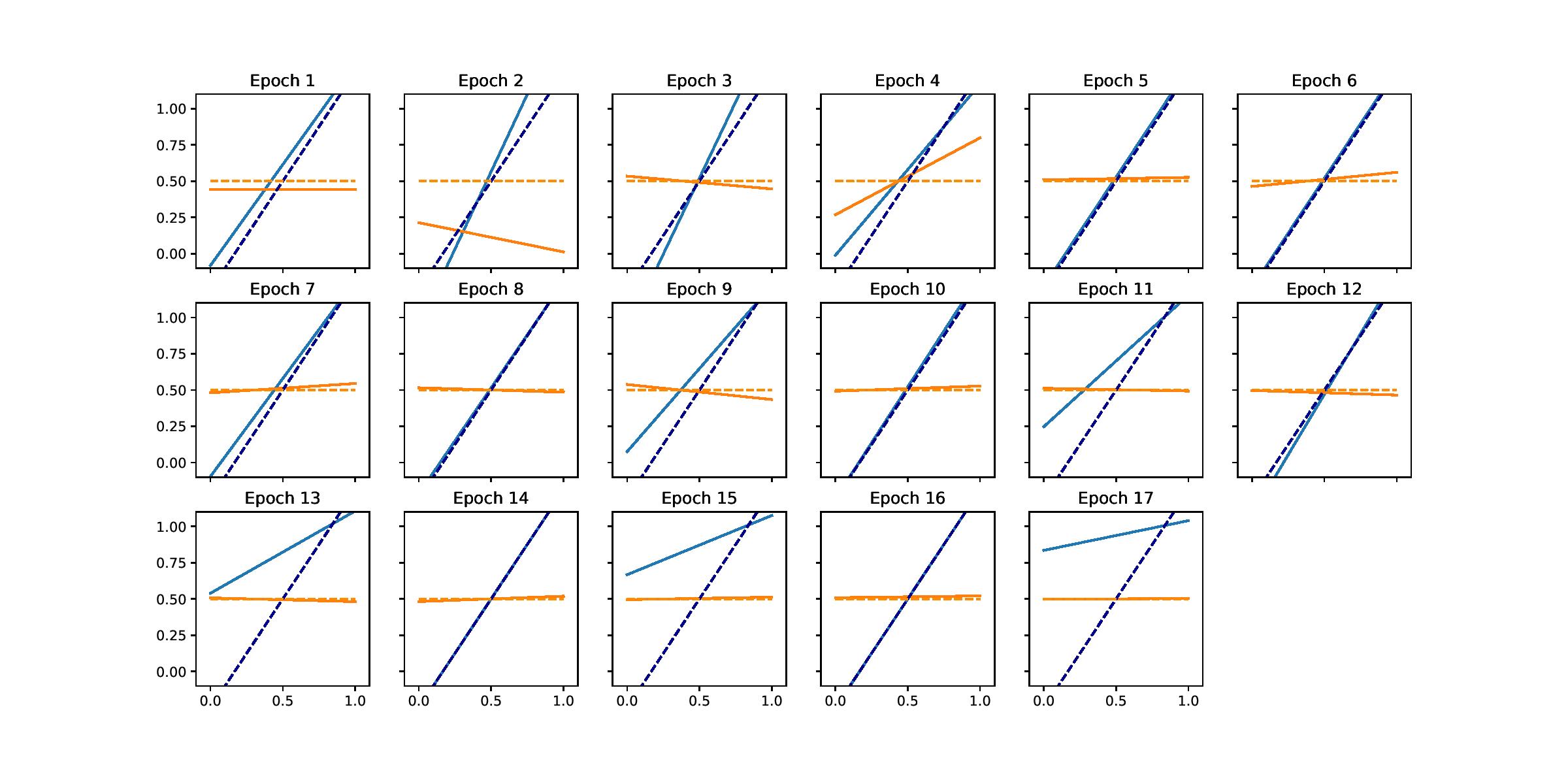}
    \caption{Linear models at the end of each epoch $m$ and action $a=1$ (blue) and $a=2$ (orange). Solid lines are fitted models $\hat{f}(\cdot, a)$. Dashed lines are oracle best linear approximation under uniform action sampling.}
    \label{fig:oscillation}
\end{figure}

The takeaway from this example is that in the presence of misspecification we must curb the amount of exploitation. This insight underpinned the construction of the algorithm presented here.

\end{document}